\documentclass[preprint,11pt]{elsarticle}
\usepackage{subcaption}
\usepackage[utf8]{inputenc}
\usepackage{amsfonts}
\usepackage{amsmath}
\usepackage{amssymb}
\usepackage{amsthm}
\usepackage{rotating}
\usepackage{graphicx}
\usepackage{color}
\usepackage{booktabs}
\usepackage{tikz}
\usetikzlibrary{decorations}
\usetikzlibrary{decorations.text}
\usetikzlibrary{decorations.pathreplacing}
\usetikzlibrary{backgrounds}
\usetikzlibrary{calc}
\usetikzlibrary{arrows,arrows.meta}
\usetikzlibrary{tikzmark,positioning}
\usetikzlibrary{shapes}
\usetikzlibrary{shapes.misc}
\usetikzlibrary{mindmap}
\usetikzlibrary{positioning}
\usetikzlibrary{math}
\usetikzlibrary{calc}
\usetikzlibrary{patterns}
\usepackage{tikz-cd}

\usepackage{verbatim}
\usepackage{hyperref}
\usepackage{bm}
\usepackage{amsmath}
\usepackage{tikz}
\usepackage{amsthm} 
\usepackage{placeins}
\newtheorem{theorem}{Theorem}[section] 
\newtheorem{corollary}[theorem]{Corollary} 

\usetikzlibrary{arrows.meta, positioning}

\newtheorem{remark}{Remark}[section]

\usepackage{multirow}

\newcommand{\vertiii}[1]{{\left\vert\kern-0.25ex\left\vert\kern-0.25ex\left\vert #1 
		\right\vert\kern-0.25ex\right\vert\kern-0.25ex\right\vert}}

\begin{document}

\begin{frontmatter}

\title{Structure-Preserving Digital Twins via Conditional Neural Whitney Forms }
\author[1]{Brooks Kinch}
\author[1]{Benjamin Shaffer}
\author[2]{Elizabeth Armstrong}
\author[2]{Michael Meehan}
\author[2]{John Hewson}
\author[1,2]{Nathaniel Trask}
\address[1]{Mechanical Engineering and Applied Mechanics, University of Pennsylvania, Philadelphia, PA, USA}
\address[2]{Sandia National Laboratories, Albuquerque, NM, USA}

\begin{abstract}
We present a framework for constructing real-time digital twins based on
structure-preserving reduced finite element models conditioned on a latent
variable \(Z\). The approach uses conditional attention mechanisms to learn both a reduced finite element basis and a nonlinear conservation law within the framework of finite element exterior calculus (FEEC). This guarantees numerical well-posedness and exact preservation of conserved quantities, regardless of data sparsity or optimization error. The conditioning mechanism supports real-time calibration to parametric variables, allowing the construction of digital twins which support closed loop inference and calibration to sensor data. The framework interfaces with conventional finite element machinery in a non-invasive manner, allowing treatment of complex geometries and integration of learned models with conventional finite element techniques.

Benchmarks include advection diffusion, shock hydrodynamics, electrostatics,
and a complex battery thermal runaway problem. The method achieves accurate
predictions on complex geometries with sparse data (25 LES simulations),
including capturing the transition to turbulence and achieving real-time inference (\(\sim 0.1\)\,s) with a speedup of \(3.1 \times 10^8\) relative to LES. An open-source implementation is available on GitHub.
\end{abstract}


\end{frontmatter}

\section{Overview and literature review}\label{sec:overview}

\subsection{Digital twins for complex physical systems.} 

Digital twins are simulators that support real-time prediction of a physical
system while simultaneously calibrating to real-time sensor data. The recent
National Academies report \cite{willcox2023foundational} identifies this
closed-loop capability as central to a digital twin:

\begin{quote}
A digital twin is a set of virtual information constructs that mimics the
structure, context, and behavior of a natural, engineered, or social system
(or system-of-systems), is dynamically updated with data from its physical
twin, has a predictive capability, and informs decisions that realize value.
The bidirectional interaction between the virtual and the physical is central
to the digital twin.
\end{quote}

While a digital twin could be as simple as a linear regression model, recent advances in data-driven techniques have significantly improved accuracy and speed of predictions. The literature may broadly be split between two approaches: projection-based reduced order models (ROMs) or other dimension-reduced FEM methods (e.g. \cite{kapteyn2020interpretable,cicci2023cardiac,pretorius2021predictive,yao2018nonintrusive}) or neural operators (e.g. \cite{chattopadhyay2024oceannet,liu2024deeponetdt,kobayashi2024deeponet}) using techniques like Fourier Neural Operators \cite{li2021fourier,patel2021physics}, DeepONets \cite{lu2021learning,goswami2022physics}, or modern transformer architectures \cite{wang2024cvit,bodnar2025earth}. Both strategies support real-time prediction but have complementary limitations. ROMs are built on a rigorous Galerkin framework and support analysis, but nonlinear and advection-dominated problems are difficult to treat, particularly when the governing equations are partially known. Neural operators are flexible and can regress a wide range of behaviors, but they are black-box models that typically do not provide guarantees of conservation or stability.

The goal of this work is to combine these benefits. We develop a
transformer-based architecture that encodes a reduced mixed finite element
space representing the geometry and a nonlinear discrete flux representing the physics, in effect merging the structure of ROMs with the flexibility of data-driven neural operators.

\subsection{Operator regression and conditional neural fields.} A number of works in scientific machine learning have considered the operator regression problem: learning of infinite-dimensional function to function maps. Motivated by the Riesz representation theorem, this task requires learning both a representation of functions in the domain and a suitable dual representation for the range. Many previous works may be understood in this context. DeepONets, separate the representation of the input function into a branch net and the spatial operator into a trunk net, effectively constructing a data-driven basis and projection \cite{lu2021learning}. Other methods explicitly construct a reduced basis with principal component analysis before encoding the operator with a neural network \cite{kovachki2023neural}. Other approaches encode operators via finite difference stencils with convolutional or graph architectures \cite{barsinai2019datadriven}.

In this work, we adopt the perspective of \textit{conditional neural fields}, an operator learning formulation originally developed in computer vision to model structured image-to-image mappings modulated by a second conditioning field \cite{xie2022neural}. From a probabilistic viewpoint, this corresponds to sampling from a data distribution conditioned on a latent or observed field, often in a generative modeling context. Architectures designed for this setting commonly employ self-attention layers to capture intra-field structure and cross-attention mechanisms to incorporate the conditioning field, providing a natural mechanism for integrating multiple modalities. Recently, these architectures have been identified as a powerful alternative for operator regression \cite{wang2024bridging} and used in auto-regressive frameworks for building physics-emulators of problems like weather near-casting and composite mechanical modeling \cite{wang2024cvit,wang2024micrometer,bodnar2025earth}. In this work, we consider an abstract
conditioning field \(Z\) that may represent a vector of sensor readings, or
other parameters influencing the finite element space and governing physics.

\subsection{Physical structure preservation with FEEC.}

For many classes of multiphysics problems, structure preservation in the form of mass, momentum, and energy conservation, as well as variational structure that supports stability and well-posedness analysis, is crucial for quantitative and reliable predictions. A large fraction of physics-informed machine learning strategies embed physics by adding PDE residuals and other desirable properties as penalty terms in the loss function. It is well known that this weak enforcement of physics can lead to difficulties in both training and accuracy, although recent work has proposed schemes that partially mitigate these effects \cite{wang2021understanding,wang2024multistage}.

An alternative trend is to design neural architectures that preserve structure by construction. These methods have been proposed in both physical and non-physical contexts. Examples include equivariant or invariant networks that encode symmetries such as rotation, translation, reflection, or gauge \cite{cohen2016group,batzner2022e3}, monotone or convex networks \cite{amos2017input,runje2022constrained,bengio2005convex,chaudhari2023learning,schotthofer2022structure}, Hamiltonian, Lagrangian, or symplectic networks \cite{greydanus2019hamiltonian,cranmer2020lagrangian,chen2019symplectic} and dissipative bracket generalizations \cite{gruber2023reversible,gruber2024efficiently},
permutation equivariance \cite{roddenberry2021principled}, topological invariants \cite{zhang2018machine,hu2019topology}, mechanical or thermodynamic principles \cite{vlassis2021sobolev,cai2021equivariant}, and tensor invariants \cite{ling2016machine,fuhg2023stress}.

In traditional modeling, finite element exterior calculus (FEEC) provides a general framework for constructing finite element spaces that preserve the topological structure of differential operators such as grad, curl, and div \cite{arnold2006finite,arnold2018finite}. In this framework, the connection between physics (expressed as exchanges of fluxes) and geometry (expressed through generalized Stokes theorems) is captured using exterior derivatives and
the de Rham complex. Preserving this topological structure leads to
discretizations that conserve invariants \cite{arnold2007compatible}, admit natural interface conditions for hybridizable schemes
\cite{arbogast2007multiscale,jiang2024structure}, and provide a rigorous and stable platform for large-scale multiphysics simulation
\cite{anderson2021mfem,rathgeber2016firedrake,bochev2012solving}.

In our previous work \cite{trask2022enforcing}, we developed a data-driven version of the exterior calculus that enables PDE-constrained optimization of graph-based networks to extract models from data. This formulation preserves the de Rham complex and includes a complete Lax-Milgram-style stability analysis. We later introduced data-driven Whitney forms that allow finite element spaces to
be learned \cite{actor2024data}, exposing a differentiable generalization of a computational mesh and providing a means of learning a structure-preserving reduced basis. In the current work, we combine these ideas to learn a parsimonious representation of geometry and physics simultaneously, identifying a set of control volumes that provide an optimal description of the dynamics.

\subsection{Primary contributions of the current work.}
This work builds upon our previous foundational development of a data-driven exterior calculus \cite{trask2022enforcing} and data-driven Whitney forms \cite{actor2024data}. While these works establish foundational techniques, they did not construct a fully nonlinear, data-driven finite element method and did not employ transformers or a conditioning mechanism, with analysis restricted to non-linear elliptic problems or to linear Hodge Laplacian problems. We highlight the primary contributions of the current work here:
\begin{itemize}
    \item Well-posedness analysis for a generic nonlinear perturbation of a Hodge Laplacian, providing justification for equality constrained optimization.
    \item Learning problem considers a corresponding class of nonlinear conservation laws with guaranteed existence and uniqueness of solutions.
    \item Whitney forms constructed via a novel conditioning mechanism to allow real time calibration of reduced finite element space and models to parametric dependencies and real time inference.
    \item A novel conditional transformer backbone supports highly accurate conditional regression with orders of magnitude accuracy demonstrated over a traditional DeepONet.
    \item The finite element construction admits a non-invasive implementation for complex geometries, allowing the technique to be adopted to mature legacy codes with modest software engineering.
    \item The learned physics representation is discretization invariant and generalizes across geometries.
    \item We provide an open-source PyTorch implementation with example code to interface with an open source conventional FEM library \cite{gustafsson2020scikit}.
    \item Construction of a digital twin for thermal runaway management of a lithium ion battery provides real-time inference of convective heat transfer calibrated to temperature sensor data, capturing the transition to turbulence and obtaining a speedup of $3.11\times10^8$ over LES simulations used to train the surrogate model.
\end{itemize}

In contrast to standard autoregressive transformers that interpolate time‐series frames, our approach yields \emph{reduced‐order finite‐element models} that integrate naturally into existing finite element workflows. One can impose \emph{initial and boundary conditions} in the reduced basis in a traditional Galerkin sense; apply \emph{classical uncertainty quantification} methods, e.g.\ Monte Carlo or stochastic collocation, on the reduced system to obtain statistical error bounds \cite{xiu2005high,nobile2008sparse}; and \emph{couple subdomain ROMs with full FEM} via non‐overlapping domain‐decomposition algorithms for multi‐component simulations \cite{jiang2024structure}. The general framework may be executed non-invasively on any underlying finite element space with positive basis functions satisfying a partition of unity property, and therefore hp-finite elements or hierarchical bases and other standard techniques for accessing improved accuracy, scalability or adaptivity may easily be integrated.


\section{Mathematical preliminaries}

We gather here fundamental definitions and key results from previous works. For an introduction to FEEC we refer readers to Whitney's original paper \cite{Whitney1957}, Arnold's overview work \cite{arnold2018finite}, as well as \cite{bochev2006principles} for a more holistic viewpoint of how many seemingly different discretizations (mimetic finite differences, staggered finite volumes, FEEC, and others) admit a unified analysis through exterior calculus. We refer to our previous works: \cite{trask2022enforcing} for an introduction to the data-driven exterior calculus and analysis of learning nonlinear conservation laws, and \cite{actor2024data} for learning Whitney forms and their construction/analysis on manifolds of arbitrary dimension. In the present work, we restrict ourselves to a simple Euclidean setting on a compact domain $\Omega \subset \mathbb{R}^d$ with Lipschitz boundary $\partial \Omega$ and outward facing normal $\hat{n}$. We denote the $L^2$-inner product on $\Omega$ by $(\cdot,\cdot)$, and the corresponding trace on $\partial \Omega$ by $<\cdot,\cdot>$. When expanding a function in a basis, we use $\hat{\cdot}$ to denote basis coefficients, e.g. $u(x) = \underset{i}{\sum} \hat{u}_i \phi_i(x)$. We regularly adopt the Einstein summation convention, in which repeated indices imply a sum (e.g. $A_{ij} x_j := \underset{j}{\sum} A_{ij} x_j$).

\subsection{Whitney forms}
Consider a \textit{partition of unity} consisting of a set of functions $\lambda_i$, $i = 1,\dots,N$, with the property $\lambda_i \geq 0$ for all $i$ and $\sum_i \lambda_i = 1$. In the typical FEEC construction, $\lambda$ corresponds to a barycentric interpolant of nodal degrees of freedom (DOFs) on simplices corresponding to the traditional $\mathcal{P}_1$ nodal FEM space. The \textit{low-order Whitney forms} correspond to a family of finite element spaces:

\begin{subequations} \label{whitneyspaces}
\begin{align}
\mathcal{W}_0 &= \operatorname{span} \big\{\, \lambda_i \big\}, \label{eq:W0} \\[6pt]
\mathcal{W}_1 &= \operatorname{span} \big\{\, \lambda_i \nabla \lambda_j - \lambda_j \nabla \lambda_i \big\}, \\[6pt]
\mathcal{W}_2 &= \operatorname{span} \Big\{\,
    \lambda_i \nabla \lambda_j \times \nabla \lambda_k 
    + \lambda_j \nabla \lambda_k \times \nabla \lambda_i 
    + \lambda_k \nabla \lambda_i \times \nabla \lambda_j
\,\Big\}, \\[6pt]
\mathcal{W}_3 &= \operatorname{span} \Big\{\,
    \begin{aligned}[t]
    &\lambda_i \nabla \lambda_j \cdot (\nabla \lambda_k \times \nabla \lambda_l)
-\lambda_j \nabla \lambda_i \cdot (\nabla \lambda_k \times \nabla \lambda_l)\\
+&\lambda_k \nabla \lambda_i \cdot (\nabla \lambda_j \times \nabla \lambda_l)
-\lambda_l \nabla \lambda_i \cdot (\nabla \lambda_j \times \nabla \lambda_k)
    \Big\}.
    \end{aligned}
\end{align}
\end{subequations}

Note that each is antisymmetric with respect to an exchange of indices. In $\mathbb{R}^3$, one may identify: $\mathcal{W}_0$ with Lagrange elements (nodal point evaluation DOFs $\mu^0_i(u) = \delta_{x_i} \circ u$); $\mathcal{W}_1$ with Nedelec elements (edge circulation DOFs $\mu^1_i(u) = \int_{e_i} \mathbf{u} \cdot d \mathbf{l}$); $\mathcal{W}_2$ with Raviart-Thomas elements (face flux DOFs $\mu^2_i(u) = \int_{f_i} \mathbf{u} \cdot d \mathbf{A}$); and $\mathcal{W}_3$ with the space of discontinuous piecewise constant functions (cell average DOFs $\mu^3_i(u) = \int_{c_i} u dV$). These finite element spaces provide conforming subspaces to corresponding $H(grad)$, $H(curl)$, $H(div)$ and $L^2$ function spaces, respectively. In the remainder, we denote as bases $\mathcal{W}_k = span\left( \psi^k \right)$. Together, they form the \textit{de Rham complex}:

\begin{equation}
\begin{tikzcd}[row sep=large, column sep=large]
H(grad) \arrow[r, "\nabla"] & H(curl) \arrow[r, "\nabla \times"] & H(div) \arrow[r, "\nabla \cdot"] & L^2 \\
\mathcal{W}_0 \arrow[u, hook] \arrow[r, "d_0"] & \mathcal{W}_1 \arrow[u, hook] \arrow[r, "d_1"] & \mathcal{W}_2 \arrow[u, hook] \arrow[r, "d_2"]  & \mathcal{W}_3 \arrow[u, hook]
\end{tikzcd}.
\end{equation}

The linear maps $d_k:\mathcal{W}_k \rightarrow \mathcal{W}_{k+1}$ denote \textit{discrete exterior derivatives} that serve as discrete analogues of continuous gradient, curl, and divergence operators. These maps are surjective ($d_k \mathcal{W}_k \subseteq \mathcal{W}_{k+1}$) and form an \textit{exact sequence} satisfying $d_{k+1} \circ d_{k} = 0$. This exactness property reflects the discrete counterpart of the familiar vector calculus identities $\nabla \times \nabla = \nabla \cdot \nabla \times = 0$, encoding topological structure and ensuring that geometric relations are consistent with the underlying conservation laws.

\subsection{The surjectivity property and conservation structure}

The surjective property of Whitney forms may be used to both discretize diffusive fluxes and to discretize conservation balances while exactly preserving conservation of fluxes. A rigorous discussion may be found in \cite{actor2024data}; here we use a simple example to illustrate how Whitney forms provide an exact treatment of div/grad/curl and conservation structure.

To illustrate the first, consider the following Galerkin expression projecting the gradient of $u \in \mathcal{W}_0$ as a field $J \in \mathcal{W}_1$, for all $v \in \mathcal{W}_1$.
\begin{equation}
    (J,v) = (\nabla u, v)
\end{equation}
By surjectivity, $J = \nabla u$ exactly, and thus $J$ represents the gradient of $u$ in a pointwise manner. Alternatively, one may observe that the degrees of freedom for $J$ may be computed from the fundamental theorem of calculus as $\mu^1_i(J) = \int_{e_i} J \cdot d \mathbf{l} = u(e_i^+) - u(e_i^-)$. To illustrate discrete conservation balance, a divergence form conservation law balancing fluxes $J$ against sources $f$ may be discretized

\begin{equation}
    \nabla \cdot J = f \quad \rightarrow \quad (J,\nabla q) = <J,q> - (f,q) \text{   for all } q \in \mathcal{W}_0.
\end{equation}
Using only properties of the partition of unity ($\sum_i \lambda_i = 1$, $\sum_i \nabla \lambda_i = 0$), and choosing $q = \lambda_i$
\begin{align}\label{antisymmflux}
    (J,\nabla \lambda_i) &= \sum_j \int \left( \lambda_i \nabla \lambda_j - \lambda_j \nabla \lambda_i \right) \cdot J dx\\
    &= \sum_j \int \psi_{ij}^1 \cdot J dx,
\end{align}
we arrive at a discrete divergence over partition $i$ in terms of discrete fluxes ($\int \psi^1_{ij}\cdot J dx$) which are antisymmetric in $i$ and $j$. Therefore, the Whitney 1-forms encode equal and opposite fluxes exchanged between partitions of space encoded by Whitney 0-forms, providing a local conservation principle; if \eqref{antisymmflux} is summed over $i$, we obtain a telescoping series so that internal fluxes cancel each other out. For further discussion regarding conservation in data-driven FEEC, see \cite{actor2024data}.

\subsection{Well-posedness of learned physics}
To guarantee well-posedness when learning nonlinear physics of unknown functional form, we will pose in the following section a conservation law consisting of a non-linear perturbation of a Laplacian. Similar to classical artificial viscosity methods, this will allow us to obtain control over the nonlinearity provided the diffusion is ``strong enough'' to dominate the nonlinearity. 

\begin{theorem}[Gustafsson \cite{gustafsson1995time}]\label{gustythm}
    Consider the nonlinear system of equations
    \begin{equation}\label{wellposedprob}
        \mathbf{A}x + \epsilon \mathbf{F}(x) = b,
    \end{equation}
    where $\epsilon>0$ and $\mathbf{F}$ is a vector-valued nonlinear function with Lipschitz constant $C_L$
    $$||\mathbf{F}(x) - \mathbf{F}(y)||_2 \leq C_L ||x-y||_2.$$
    Define $\tau = \epsilon C_L ||A^{-1}||$. If $\tau < 1$, then \eqref{wellposedprob} has a unique solution.
\end{theorem}
\begin{proof}
    See \cite{gustafsson1995time}, Appendix A.3.
\end{proof}
\begin{corollary}
Assume $\mathbf{A}$ is invertible and satisfies the Poincare-like inequality
\begin{equation}
    ||x||_2^2 \leq C_p x^\intercal \mathbf{A} x
\end{equation}
then $\tau \leq \epsilon C_p C_L$, and following Theorem \ref{gustythm}, \eqref{wellposedprob} has a solution if $\epsilon C_p C_L < 1$.
\end{corollary}
\begin{proof}
    See Horn and Johnson \cite{horn2012matrix} for a discussion of Loewner ordering and monotonicity of the matrix inverse. By Loewner ordering, 
    $$ x^\intercal x \leq C_p x^\intercal \mathbf{A} x \implies x^\intercal \mathbf{A}^{-1} x \leq C_p x^\intercal x,$$
    and we can bound $||\mathbf{A}^{-1}|| \leq C_p$.
\end{proof}

In the context of Whitney forms, we construct an $\mathbf{A}$ meeting these conditions by identifying a discrete Laplacian operator with an associated Poincare inequality.

\begin{theorem}
    Denote by $\delta_{i,jk} = \delta_{ij} - \delta_{ik}$ the graph gradient operator mapping $\mathbb{R}^{N}\rightarrow \mathbb{R}^{N\times N}$, and $\mathbf{M}_i$ the mass matrix associated with $\mathcal{W}_i$. Then the following identities hold, for $q,u \in \mathcal{W}_0$ and $v,J \in \mathcal{W}_1$.
    $$(v,\nabla u) = \hat{v}^\intercal \mathbf{M}_1 \delta \hat{u}$$
    $$-(J,\nabla q) = \hat{q}^\intercal \delta^\intercal \mathbf{M}_1 \hat{J}$$
    Consider the Poisson equation in mixed form, namely, find $(u,J) \in \mathcal{W}_0\times \mathcal{W}_1$ such that for any $(q,v) \in \mathcal{W}_0\times \mathcal{W}_1$
    $$ (J,v) + (\nabla u,v) = 0 $$
    $$ (J,\nabla q) = <J,q> + (f,q).$$
    If $J \cdot \hat{n}|_{\partial\Omega} = 0$, the resulting system of equations
    $$\mathbf{L} \hat{u} := \delta^\intercal \mathbf{M}_1 \delta \hat{u} = \mathbf{M}_0 \hat{f}$$
    has a unique solution modulo a constant vector in the null-space, and the stiffness matrix $\mathbf{L}$ corresponding to the discretized Hodge Laplacian is symmetric positive definite with a Poincare inequality $$\hat{u}^\intercal \mathbf{M}_0 \hat{u} \leq C_p(h)  \hat{u}^\intercal \mathbf{L} \hat{u}.$$ 
\end{theorem}
\begin{proof}
For the first identity, $(v,\nabla u)=\sum_{ij}\hat{v}_i (\psi^1_i,\nabla \lambda^0_j) \hat{u}_j = \sum_{ijk}\hat{v}_i (\psi^1_i,\lambda^0_k\nabla \lambda^0_j-\lambda^0_j\nabla \lambda^0_k) \hat{u}_j = \hat{v}^\intercal \mathbf{M}_1 \delta \hat{u}$, with the second following similarly. The definition of the discrete Laplacian $\mathbf{L}$ follows from direct substitution of the identities into the bilinear form. Well-posedness follows by identifying $\mathbf{L}$ as a weighted graph Laplacian with weights given by $\mathbf{M}_1$. The remaining proof follows from \cite{trask2022enforcing}, taking $\mathbf{B}_0$ and $\mathbf{B}_1$ as identity matrices, $\mathbf{D}_0$ and $\mathbf{D}_1$ as $\mathbf{M}_0$ and $\mathbf{M}_1$ and applying Theorem 3.4.
\end{proof}
\begin{remark}
    The above matrix does not precisely satisfy the conditions of Theorem \ref{gustythm}, as the matrix is only positive semidefinite and the inverse does not exist; the Laplacian has a constant vector in its null-space. Formally, the Loewner ordering only holds for the \textit{pseudo-inverse}, $\mathbf{A}^\dagger \leq C_p I$. In practice, we will need to stabilize the null-space of $\mathbf{A}$ to guarantee stability.
\end{remark}

\subsection{Construction of coarse-grained de Rham complex}
In contrast to traditional FEEC in which Whitney forms are constructed by first defining $\lambda$ as a barycentric interpolant and then constructing spaces via \eqref{whitneyspaces}, our work instead uses a transformer to construct a space of Whitney 0-forms \eqref{eq:W0} which possess the necessary partition of unity property. The construction will use the following fact that partitions of unity are closed under convex combinations.

\begin{theorem}\label{convexpouthm}
    Let $\lambda_i$, $i = 1,\dots,N$ define a partition of unity satisfying $\lambda_i \geq 0$ for all $i$ and $\sum_i \lambda_i = 1$. Let
    \begin{equation}
        \psi^0_i = \sum_a W_{ia} \lambda_a (x),
    \end{equation}
    where the convex combination tensor $W$ has positive entries ($W_{ia}>0$) and unity row sum ($\underset{i}{\sum}W_{ia} = 1$). Then the space $\mathcal{W}_0 = \operatorname{span} \big\{\, \psi^0_i \big\}$ forms a partition of unity.
\end{theorem}
\begin{proof}
    Trivially, $\psi^0_i \geq 0$ by positivity of $\lambda$ and $W$. Then, $\underset{i}{\sum}\psi^0_i = \underset{ia}{\sum}W_{ia} \lambda_a = \underset{a}{\sum} \left(\underset{i}{\sum}W_{ia} \right)\lambda_a = \underset{a}{\sum} \lambda_a =1$.
\end{proof}
We refer to $\lambda$ and $\psi^0$ as \textit{fine-scale} and \textit{coarse-scale} Whitney forms, respectively. While our previous work focused on learning the entries of $W$ directly \cite{actor2024data}, the current approach uses a cross-attention transformer to modulate $W$ in response to the conditioning variable $Z$, providing finite element spaces that adapt directly to the conditioning variable while ensuring through \eqref{whitneyspaces} that the downstream de Rham complex is appropriately constructed. As shown in \cite{trask2022enforcing}, the $\mathcal{W}_1$ and $\mathcal{W}_2$ spaces can be used for preserving structure related to the involution condition in electromagnetism, the current work will exclusively work with $\mathcal{W}_0$ and $\mathcal{W}_1$ for simplicity.

\section{Learning Problem}

To treat boundary conditions, we decompose the Lipschitz boundary $\partial \Omega = \Gamma_N \cup \Gamma_D$, where $\Gamma_N$ and $\Gamma_D$ are disjoint partitions associated with Neumann and Dirichlet conditions, respectively. We assume that our data satisfies the conservation law
\begin{equation}\label{eq:strongforma}
    \partial_t u_\theta + \nabla \cdot \mathbf{F}_\theta = f,
\end{equation}
where $u_\theta \in \mathbb{R}^N$ is a density of $N$ conserved quantities and $\mathbf{F}_\theta$ is a corresponding flux. We denote by $\theta$ the dependency of these terms on both machine-learnable parameters and the conditioning variable $Z$; for brevity we will omit $\theta$ dependency unless necessary. We similarly denote the convex combination tensor in Theorem \ref{convexpouthm} as $W_\theta$ when it is trainable. The conservation law ansatz yields the following global conservation principle,
\begin{equation}
    \frac{d}{dt} \int_\Omega u_\theta dx = \int_\Omega f dx - \int_{\partial \Omega} \mathbf{F} \cdot d\mathbf{A},
\end{equation}
and thus integrals of $u$ are preserved in the setting where $f = 0$ and $\mathbf{F}\cdot \mathbf{n}|_{\partial \Omega} = 0$. All learnable physics are embedded in the flux equation
\begin{equation}\label{eq:strongformb}
    \mathbf{F} = - \epsilon \nabla u + \mathcal{N}[u;\theta]
\end{equation}
where $\mathcal{N}$ is a generic nonlinear operator with no assumed special structure, and $\epsilon$ denotes a trainable amplitude of the Laplacian which will stabilize the learned physics.

To obtain a discrete system, we consider trainable \eqref{whitneyspaces}, $\mathcal{W}_0^\theta$ and $\mathcal{W}_1^\theta$, and define the following mixed Galerkin form for \eqref{eq:strongforma} and \eqref{eq:strongformb}. Let $(u,\mathbf{F}) \in \mathcal{W}_0^\theta  \times \mathcal{W}_1^\theta$ such that for all $(q,\mathbf{v}) \in \mathcal{W}_0^\theta  \times \mathcal{W}_1^\theta$
\begin{subequations} \label{eq:WeakForm}
\begin{align}
\frac{d}{dt} (u,q) - (\mathbf{F}, \nabla q) &= (f, q) + \langle \mathbf{F}_N, \nabla q \rangle, \label{eq:WeakFormA} \\[6pt]
(\mathbf{F}, \mathbf{v}) &= (\nabla u, \mathbf{v}) + (\mathcal{N}[u], \mathbf{v}). \label{eq:WeakFormB}
\end{align}
\end{subequations}
where we assume Neumann conditions $\mathbf{F}\cdot \mathbf{n}|_{\Gamma_N} = F_N$, and Dirichlet conditions may be imposed by a standard lifting procedure, decomposing the solution into a homogeneous and inhomogeneous part $u = u_0 + u_D$, where $u_0|_{\Gamma_D} = 0$, so that $u|_{\Gamma_D} = u_D$. For simplicity of notation, we drop the subscript on $u_0$ and lump body forces and boundary condition contributions into a generic right-hand side $\mathbf{f}_\theta$, highlighting the fact that parametric forcing may be considered.

Following discretization, we obtain the dynamics
\begin{subequations} \label{eq:discreteWeakForm}
\begin{align}
\frac{d}{dt} \mathbf{M}_0 \hat{u} + \delta^\intercal_0 \mathbf{M}_1 \hat{F} = \mathbf{f}_\theta, \label{eq:discreteWeakForm1} \\[6pt]
\mathbf{M}_1 \hat{F} + \mathbf{M}_1 \delta_0 \hat{u} + \mathbf{M}_1 \mathcal{NN}[u;\theta,Z] = 0.
\label{eq:discreteWeakForm2}
\end{align}
\end{subequations}
where we have replaced the nonlinear term with a neural network that maps directly from zero-form degrees of freedom to one-form degrees of freedom, 
\begin{equation}\label{nonlinearity}
    \mathcal{NN}: \mathbb{R}^{dim(\mathcal{W}_0)\times N} \rightarrow \mathbb{R}^{dim(\mathcal{W}_1)\times N}.
sho\end{equation}
We postpone a precise specification of $\mathcal{NN}$ to the following section. By eliminating $\hat{F}$, we obtain the final discretized form
\begin{equation}\label{eq:discrete3}
\mathcal{F}(\hat{u};\theta,Z) := \delta_0^\intercal \mathbf{M}_1 \delta_0 \hat{u} + \delta_0^\intercal \mathbf{M}_1 \mathcal{NN}[\hat{u}] = \mathbf{f}_\theta.
\end{equation}

To construct a loss function for training, we pose an equality constrained quadratic program aligning the solution $u$ toward a target solution $u^{target}$ specified on a point cloud.
\begin{equation}
\underset{\hat{u},\lambda,\theta}{\text{argmin}} \mathcal{L} = \underset{\hat{u},\lambda,\theta}{\text{argmin}} \frac12 \sum_d ||u(x_d) - u^{target}(x_d)||^2 + \lambda^\intercal \left[ \mathcal{F}(\hat{u};\theta,z) - \mathbf{f}_\theta\right],
\end{equation}
where $\lambda \in \mathbb{R}^{dim(W^0)}$ is a Lagrange multiplier. We note that many different alternative reconstruction losses may be employed (e.g. directly reconstructing Whitney form DOFs, targeting flux degrees of freedom, etc.).

The Karush-Kuhn-Tucker conditions yield the necessary conditions at a local minimizer,
\begin{equation}
    \delta_{\hat{u}} \mathcal{L} = \delta_\lambda \mathcal{L} = \delta_\theta \mathcal{L} = 0.
\end{equation}

On the $k^{th}$ step of training, given a dataset $\mathcal{D} = (u^{\text{target}}_s,Z_s)_{s=1}^{N_{\text{solutions}}}$ thus consists of selecting a random solution index $s$ and sequentially solving the forward and adjoint problems for the current iterate of the learned model:
\begin{subequations} \label{eq:forwardadjoint}
\begin{align}
\mathcal{F}[\hat{u}_s^{(k)};\theta^{(k)},Z_s] = \mathbf{f}_\theta, \label{eq:forwardadjoint1} \\[6pt]
\left( \partial_{\hat{u}} \mathcal{F}\left[\hat{u}_s^{(k)};\theta^{(k)},Z_s\right]\right)^\intercal \lambda_s^{(k)} = -\sum_d\left(\hat{u}_s \cdot \psi_0(x_d)  - u^\text{target}(x_d)\right).
\label{eq:forwardadjoint2}
\end{align}
\end{subequations}

After obtaining $\hat{u}_s^{(k)}$ and $\lambda_s^{(k)}$, the model may be updated by employing a gradient-based optimizer to update $\theta^{(k)}$. In the current work, we employ the Shampoo optimizer \cite{gupta18a,shi2023distributed}. Automatic differentiation may be used to compute all required terms, including the Jacobian necessary to solve the forward problem with a Newton method and the adjoint matrix, though special care must be taken to exploit sparsity and achieve an efficient implementation.

\begin{remark}
    In light of Theorem \ref{gustythm}, we can guarantee for sufficiently small nonlinearity that a unique solution exists to the problem, and therefore that the feasible set of solutions satisfying the constraint in \eqref{eq:forwardadjoint} is nonempty, justifying our use of constrained optimization. As discussed in our prior work \cite{trask2022enforcing} for a smaller class of nonlinear elliptic problems, one could in principle estimate Poincare and Lipschitz constants and introduce constraints to guarantee the conditions of Theorem \ref{gustythm} are met. Empirically, when solving the learned systems with a standard Newton solve with backtracking \cite{armijo1966minimization} we are reliably able to obtain solutions over the entire course of training. As an alternative to using Lagrange multipliers in \eqref{eq:forwardadjoint2} in a full-space approach, one could instead adopt a reduced-space approach and back-propagate directly through the solution \cite{hinze2008optimization,biros2005parallel}.
\end{remark}

\begin{remark}[Non-invasive implementation]
In light of Theorem \ref{convexpouthm}, if $\mathcal{W}_0$ is constructed as a convex combination of Lagrange elements on simplices, the corresponding mass matrix of 0-forms may be calculated as $\mathbf{M}_0 = W^\intercal \mathbf{M}_{P_1} W$, where $\mathbf{M}_{P_1}$ is the fine-scale mass matrix which may be precomputed and stored. 

\begin{align*}
    \mathbf{M}_0 & = (\psi^0_i,\psi^0_j) = (W_{ia} \lambda_a,W_{jb} \lambda_b) = W_{ia} M_{ab} W_{jb}
\end{align*}

Similarly, the coarse-scale mass matrix for $1$-forms is obtained
from the fine-scale Nédélec mass matrix $M_{\mathrm{Ned}}$ via
\[
\mathbf{M}_1 = (W \otimes W)^\top \, \mathbf{M}_{\mathrm{Ned}} \, (W \otimes W),
\]
where $(W\otimes W)_{(pq),(ab)} = W_{pa} W_{qb}$ maps a coarse edge $(p,q)$
to its corresponding fine-scale edges $(a,b)$.
Here
\(
\mathbf{M}_{\mathrm{Ned}} = \bigl( \lambda_a \nabla \lambda_b - \lambda_b \nabla \lambda_a,
\lambda_c \nabla \lambda_d - \lambda_d \nabla \lambda_c \bigr)
\)
is the fine-scale Nédélec mass matrix.
In this manner, the coarsening for $k$-forms involves a $k$-fold Kronecker product,
so that all required mass and adjacency matrices in \eqref{eq:discrete3}
may be \textit{non-invasively} precomputed and stored, bypassing the need for quadrature. Therefore, any references to the underlying mesh and finite element space may be treated in a preprocessing step so that training consists of contractions of $W$ against sparse mass matrices.
\end{remark}
\subsection{Dirichlet boundary condition lift}\label{liftingsection}

To perform the Dirichlet condition lift alluded to earlier, we must decompose $u = u_0 + u_D$, where $u_0|_{\Gamma_D} = 0$, so that $u|_{\Gamma_D} = u_D$. We do this by partitioning $\mathcal{W}_0$ into subspaces supported exclusively on the boundary and the interior. We adopt the strategy from our previous work \cite{actor2024data} where we partition the convex combination tensor into a block diagonal structure. Let $N_{int}$ and $N_{bnd}$ denote the number of internal and boundary nodes, respectively, of the fine scale mesh, and let $M_{int}$ and $M_{bnd}$ denote the number of \textit{coarse} internal and boundary partitions, respectively. Then the convex combination matrix may be decomposed into
\begin{equation}
W= \begin{bmatrix} {W}^{\text{int}} & 0 \\ 0 & {W}^{\text{bnd}} \end{bmatrix},
\end{equation}
where ${W}^{\text{int}} \in \mathbb{R}^{M_{int}\times N_{int}}$ and ${W}^{\text{bnd}} \in \mathbb{R}^{M_{bnd}\times N_{bnd}}$. The decomposed subspaces may be defined by identifying $u_0$ with the first $M_{int}$ rows of W and $u_D$ with the last $M_{bnd}$ rows. Further explanation and visualization of internal and boundary Whitney forms may be found in \cite{actor2024data}.

\section{Transformer architecture}

The framework relies on construction of a transformer architecture which uses cross-attention to condition on $Z$ to build both $W_\theta$ (from Theorem \ref{convexpouthm}) and $\mathbf{F}_\theta$ (from \eqref{nonlinearity}). Specifically, we use a decoder-style transformer architecture composed entirely of cross-attention blocks, as illustrated in Figure \ref{fig:diagrams_grouped}. This consists of $N_{\text{blocks}}$ repeated cross attention blocks, which provides a single parameter to increase model capacity. This baseline architecture is used to construct both $W_\theta$ and $\mathbf{F}_\theta$ (See Figure \ref{fig:architecture_diagrams}). Further details regarding the specifics of the architecture, its relation to the broader transformer literature, and choices of hyperparameters can be found in \ref{app:nn_details}. While transformers are notoriously parameter-intensive, we remark that our architecture is comparatively parameter-efficient with all benchmarks using $400K-1.2M$ parameters. For comparison, general-purpose transformers such as BERT or ViT typically use ~100M+ parameters \cite{devlin2019bert,dosovitskiy2021image}, while operator-learning models for PDE surrogates like DeepONets or Fourier Neural Operators often range from $1\text{-}5M$ \cite{li2021fourier, wang2021learning,lu2021learning}. 

A key distinction between our version and the standard implementation \citep{AIAYN} is the lack of a self-attention block: this allows us to generate shape functions with very fine mesh inputs, since we avoid the standard $\mathcal{O}(N_\mathrm{nodes} \times N_\mathrm{nodes})$ self-attention complexity \cite{AIAYN} by instead employing an $\mathcal{O}(N_\mathrm{nodes} \times N_{z_\mathrm{tokens}})$ cross-attention operation. In principle, we can choose any value for $N_{z_\mathrm{tokens}}$, depending on how we choose to embed the conditioning information. In this paper, we set $N_{z_\mathrm{tokens}} = 1$ for simplicity.

\begin{figure}[t]
  \centering

  \begin{subfigure}[t]{0.48\linewidth}
    \centering
    \begin{minipage}[t][10.5cm][t]{\linewidth}
      \includegraphics[width=\linewidth]{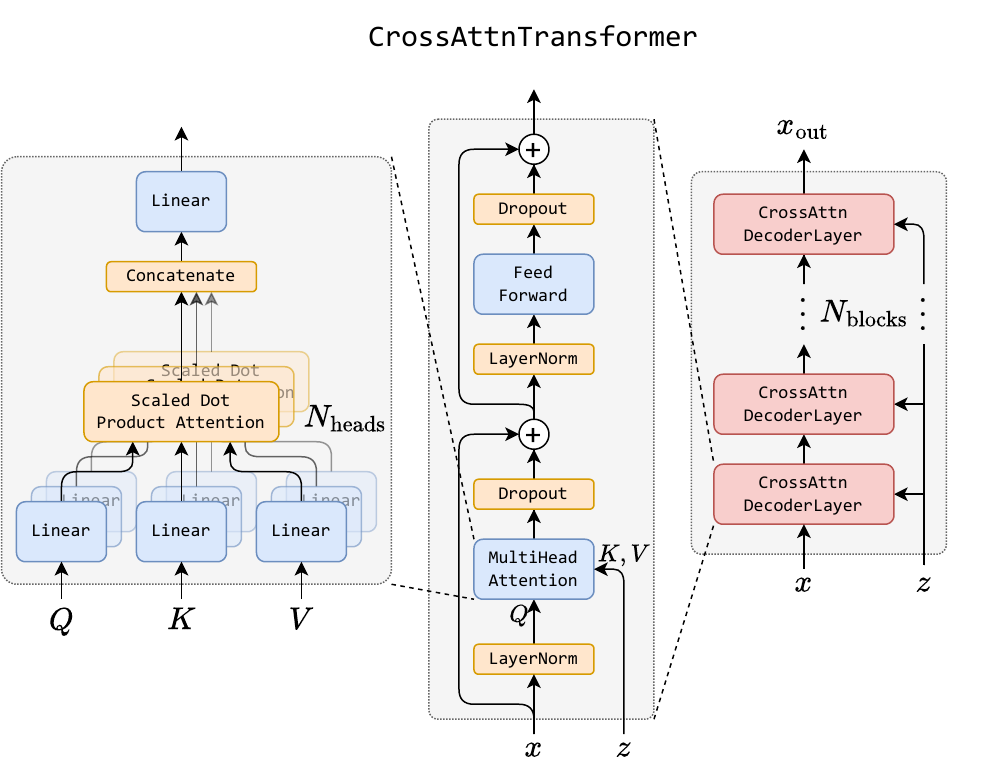}
      \caption{
        The multi-head attention implementation uses the built-in PyTorch module \texttt{torch.nn.MultiheadAttention}. The middle block is one layer of our \texttt{CrossAttnTransformer}: a standard decoding transformer without a self-attention component. Layer normalization is applied only over the final (per-token) dimension. The feedforward network is a two-layer MLP that expands and contracts the embedding dimension by a factor of 2, with a dropout layer after its activation. The \texttt{CrossAttnTransformer} is simply $N_\mathrm{blocks}$ of these layers stacked in sequence, each conditioned upon the same $z$.
      }
      \label{fig:diagrams_grouped}
    \end{minipage}
  \end{subfigure}
  \hfill
  \begin{subfigure}[t]{0.44\linewidth}
    \centering
    \begin{minipage}[t][10.5cm][t]{\linewidth}
      \includegraphics[width=\linewidth]{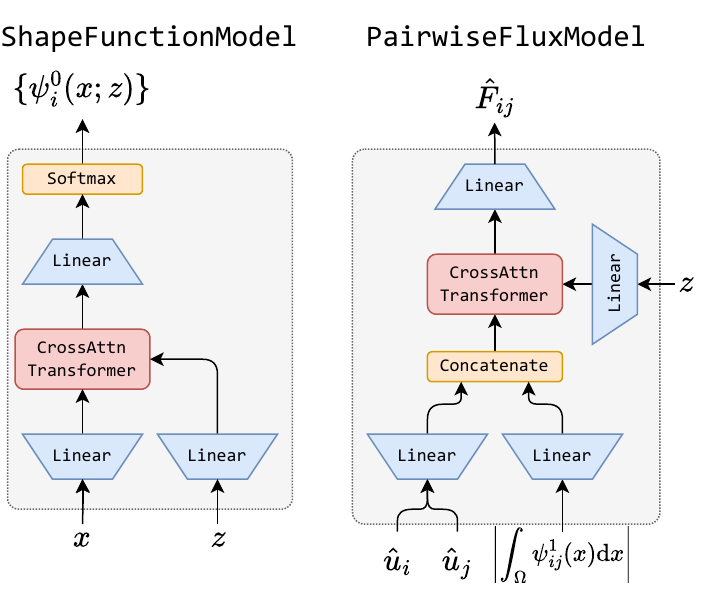}
      \caption{
        Left: The \texttt{ShapeFunctionModel} decodes a spatial coordinate $x$, given a conditioning parameter $Z$, into a set of $N_\mathrm{POU}$ coarse shape function values at $x$ that sum to unity by construction; these values prescribe $W_{ia} = \texttt{ShapeFunctionModel}_i (x_a;Z)$ from Theorem \ref{convexpouthm} when the neural operator is evaluated on nodes. Right: The \texttt{PairwiseFluxModel} constructs a latent representation of each state variable $\hat{u}$, then creates tokens by pairing these latent representations with a measure of the area between partitions; these tokens are decoded with respect to $Z$ into the flux coefficient between the two partitions. The same architecture is \textit{weight-shared} across all partition interfaces.
      }
      \label{fig:architecture_diagrams}
    \end{minipage}
  \end{subfigure}

  \caption{
    Overview of our transformer architecture. 
    (a) Multi-head cross-attention block and \texttt{CrossAttnTransformer}. 
    (b) \texttt{ShapeFunctionModel} and \texttt{pairwiseFluxModel} constructed using stacked \texttt{CrossAttnTransformer} blocks, providing \textit{resolution-independent} parameterizations of POUs.
  }
  \label{fig:transformer_overview}
\end{figure}

The neural network used to produce the shape functions $\{\psi^0_i(x; z)\}$, \texttt{ShapeFunctionModel} is evaluated on each node of the fine mesh in order to produce the transformation matrix $W$; i.e., $W_{ij} = \psi^0_i(x_j; z)$, where $i$ indexes the coarse shape functions and $j$ indexes the nodes of the fine mesh. We emphasize that the output of the neural network varies continuously with $x$; it is not confined to the nodes of the fine mesh. In practice, the same underlying problem geometry could be re-meshed or iteratively refined, and the coarse finite element basis learned on a different mesh would still work. In other words, our method is \textit{discretization-independent}. 

The model used for the flux, \texttt{PairwiseFluxModel}, maps the state variables on pairs of partitions, $\hat{u}_i$ and $\hat{u}_j$, to the flux between them, $\hat{F}_{ij}$. One single network is applied across all interfaces in a \textit{weight-sharing} configuration similar to CNNs \cite{lecun1998gradient,goodfellow2016deep}. To generalize across interfaces of varying surface area, we augment the inputs with a measure of the directed area between partitions. In concert, these features provide both desirable sample complexity, as the parameters do not scale with the number of partitions or interfaces, and identification of generalizable physics, as the same flux-state relationship must generalize across all interfaces as geometry evolves during training. 

%
\begin{figure}[t]
  \centering

  \begin{subfigure}[t]{0.6\linewidth}
    \centering
    \includegraphics[width=\linewidth]{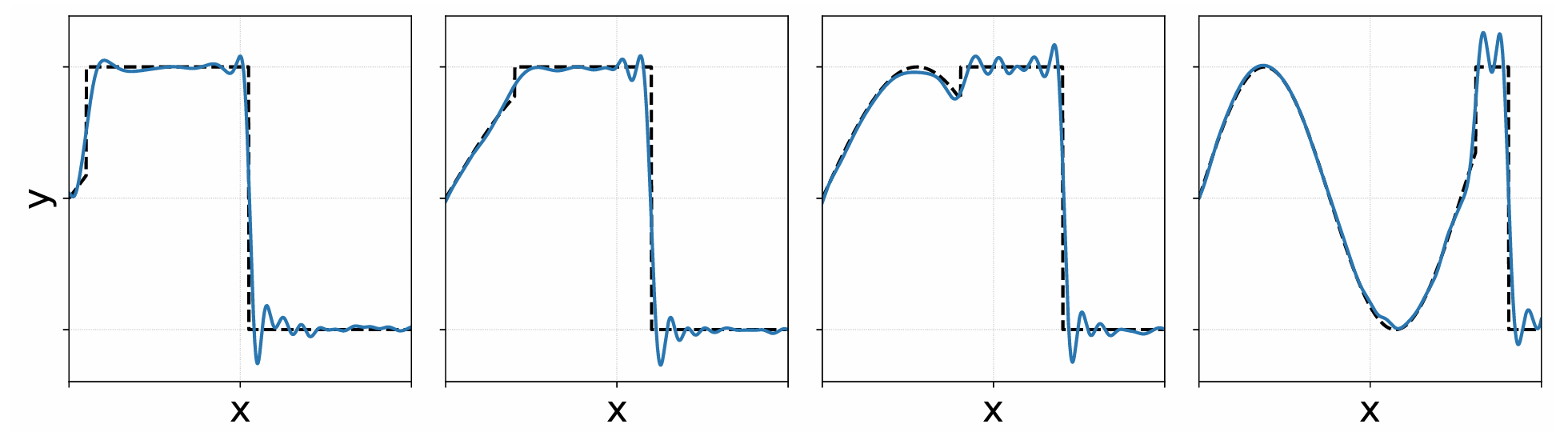}
    \caption{Representative reconstructions for a vanilla DeepONet model (395,776 parameters).}
    \label{fig:regression_deepo}

    \vspace{0.5em}

    \includegraphics[width=\linewidth]{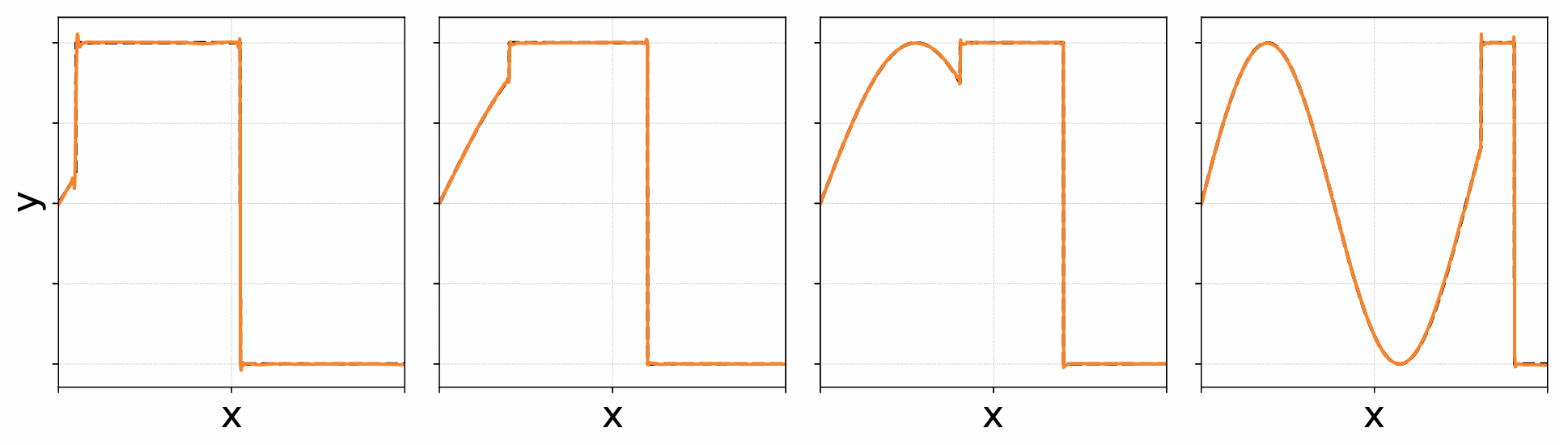}
    \caption{Representative reconstructions for proposed cross attention transformer (398,081 parameters).}
    \label{fig:regression_transformer}
  \end{subfigure}
  \hfill
\begin{subfigure}[t]{0.33\linewidth}
  \centering
  \raisebox{-3.5cm}{%
    \includegraphics[width=\linewidth]{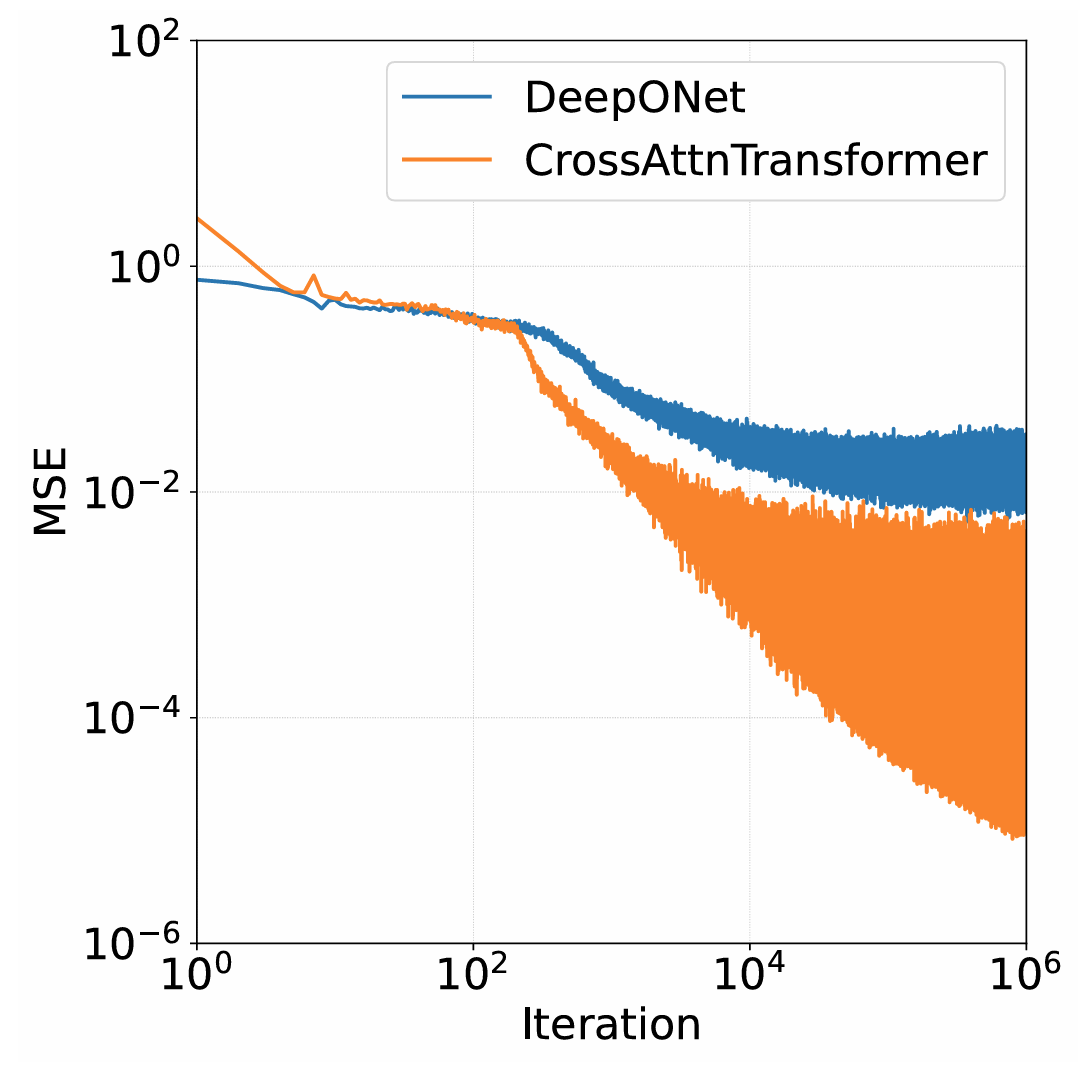}
  }
  \caption{Convergence during training for both, demonstrating a three order-of-magnitude reduction in loss in the infinite data limit.}
  \label{fig:regression_convergence}
\end{subfigure}

  \caption{
    Informal comparison of performance for a conditional regression benchmark between the proposed conditioning architecture and a vanilla DeepONet.  Representative plots spanning discontinuous to continuous data are shown in (a,b) for $Z \in \left\{
\begin{bmatrix} 3.45 \\ 0.05 \end{bmatrix},
\begin{bmatrix} 4.41 \\ 0.20 \end{bmatrix},
\begin{bmatrix} 5.68 \\ 0.40 \end{bmatrix},
\begin{bmatrix} 8.21 \\ 0.80 \end{bmatrix}
\right\}$ with true and reconstructed profiles as dashed and solid lines, respectively. The parameters for both methods are comparable and relatively small.
  }
  \label{fig:regression_benchmark}
\end{figure}

To give a general sense of the performance for the cross-attention architecture in a generic regression tasks, we evaluate both on a benchmark in Figure \ref{fig:regression_benchmark} where data is sampled from the function
\begin{equation}
f(x) = 
\left\{
\begin{array}{ll}
\sin(z_1 x) & \text{if } x < z_2 \\[0.5em]
1 & \text{if } z_2 \leq x < z_2 + \dfrac{1 - z_2}{2} \\[0.5em]
-1 & \text{if } x \geq z_2 + \dfrac{1 - z_2}{2}
\end{array}.
\right.
\end{equation}
Here, $z_1$ and $z_2$ serve as conditioning parameters to modulate the characteristic lengthscale of a smooth and piecewise constant part of a one-dimensional function, highlighting the ability to capture both smooth and discontinuous data. We compare our proposed architecture against a vanilla DeepONet \cite{lu2021learning} with a similar number of parameters in the infinite data limit; at each epoch of training, 1024 points are sampled from the unit interval conditioned upon a $Z$ sampled from the range $[A,B]$. This allows a rough characterization of the capacity of each architecture, where we observe a marked improvement from the transformer. While not a comprehensive comparison, this highlights that the transformer provides a highly accurate conditioning mechanism with a comparable model size to state-of-the-art architectures; the proposed architecture is substantially smaller than those used in contemporary large language models.

\section{Computational results}
\label{sec:results}

In the remainder, we provide a series of simple benchmarks to illustrate the performance of different aspects of the framework before constructing a digital twin.
\begin{itemize}
    \item Section \ref{exp:1dad} highlights a simple 1D advection diffusion case, showing how both physics and shape functions adapt to a boundary layer through the conditioning mechanism.
    \item Section \ref{exp:bellAD} extends this to an unstructured two-dimensional mesh, conditioning advection diffusion on the direction of flow, demonstrating the ability to generalize physics as partition geometries evolve under conditioning.
    \item Section \ref{exp:shocks} considers a Riemann problem from shock hydrodynamics while conditioning on material parameters, illustrating the ability to learn vector-valued nonlinear conservation laws with shocks.
    \item Section \ref{exp:conducting shell} considers electrostatics for a point charge in a conducting cylinder conditioning on the charge location, highlighting preservation of conservation structure to machine precision.
    \item Section \ref{exp:digitaltwin} constructs a digital twin of heat-transfer from a battery pack undergoing thermal runaway, achieving a real-time turbulence model able to capture the transition to turbulence.
\end{itemize}

For all problems, hyperparameters and other details necessary for reproduction of results may be found in \ref{app:nn_details}. We stress that the dimension of the learning dynamics for each problem scales with the number of partitions for a given problem and not with the underlying mesh; for a problem with four functions in $\mathcal{W}_0$, a small system of four non-linear equations needs to be solved. For all problems considered here, the conditioning process and solution of the consequent forward problem can be solved in under a second on a consumer-grade GPU. In each case, the number of partitions poses a hyperparameter which must be selected similar to the number of elements in a FEM simulation or the width/depth of a dense network; we present the smallest number of partitions able to obtain an accurate solution to achieve the fastest reduced model.

\subsection{1D Advection-Diffusion}\label{exp:1dad}

We consider the one-dimensional steady-state advection-diffusion equation
\begin{equation}
    \frac{d u}{dx} -{\varepsilon}\frac{d^2 u}{dx^2} = 0
\end{equation}
subject to Dirichlet boundary conditions \( u(0) = 1 \), \( u(1) = 0 \), where \( \varepsilon = 1/\mathrm{Pe} \) is the inverse Peclet number. The Peclet number \( \mathrm{Pe} \) characterizes the relative strength of advection to diffusion. In the advection-dominated regime (\( \mathrm{Pe} \to \infty \)), the problem becomes \emph{singularly perturbed}, and the solution exhibits a thin boundary layer of width \( \mathcal{O}(\varepsilon) \) near \( x = 1 \). The exact solution is
\begin{equation}\label{eqn:1dADexact}
    u(x) = 1 - \frac{1 - e^{x/\varepsilon}}{1 - e^{1/\varepsilon}},
\end{equation}
which clearly illustrates the sharp transition in the boundary layer.

Standard Galerkin finite element discretizations applied to this problem are known to produce spurious oscillations near the outflow boundary when the mesh does not adequately resolve the boundary layer. This is a classical manifestation of instability in the presence of dominant advection and motivates the use of stabilization techniques such as streamline upwind Petrov-Galerkin (SUPG) \cite{brooks1982streamline} or other stabilizations \cite{Houston2002,Schotzau2000,Guzman2006}. Without such techniques, achieving numerical stability typically requires mesh spacing \( h \ll \varepsilon \), which is prohibitively expensive for small \( \varepsilon \). By conditioning on $\varepsilon$, we can demonstrate our ability to adaptively construct compact bases tailored to resolve boundary layer physics without a prohibitive resolution requirement. Results for this problem are shown in Figure \ref{fig:1d_ad_training_and_validation_data}.
\begin{figure}[t]
  \centering

  \begin{subfigure}[t]{0.66\linewidth}
    \centering
    \begin{minipage}[t][5cm][t]{\linewidth}
      \includegraphics[width=\linewidth]{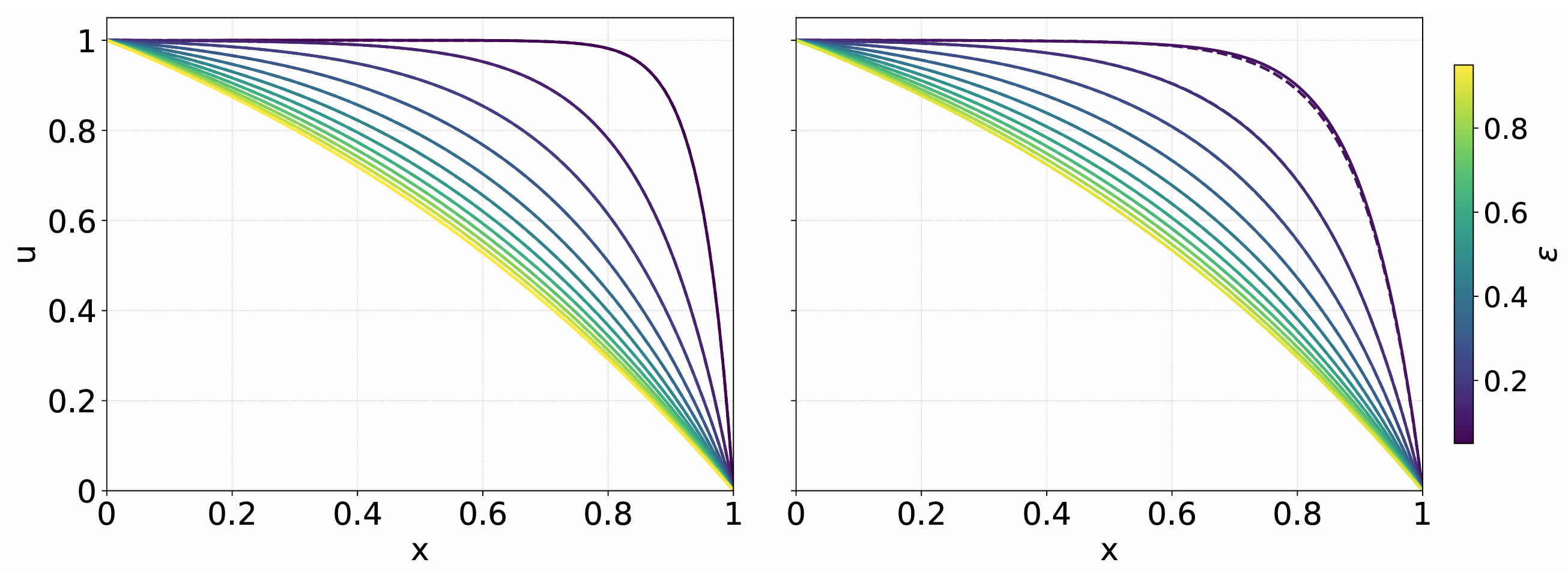}
      \caption{
        Training (left) and validation (right) solution examples for the 1D advection-diffusion problem at various $\varepsilon$. Dashed lines indicate analytic solutions, solid lines show modeled solutions.
      }
      \label{fig:1d_ad_training_and_validation_data}
    \end{minipage}
  \end{subfigure}
  \hfill
  \begin{subfigure}[t]{0.3\linewidth}
    \centering
    \begin{minipage}[t][5cm][t]{\linewidth}
      \includegraphics[width=\linewidth]{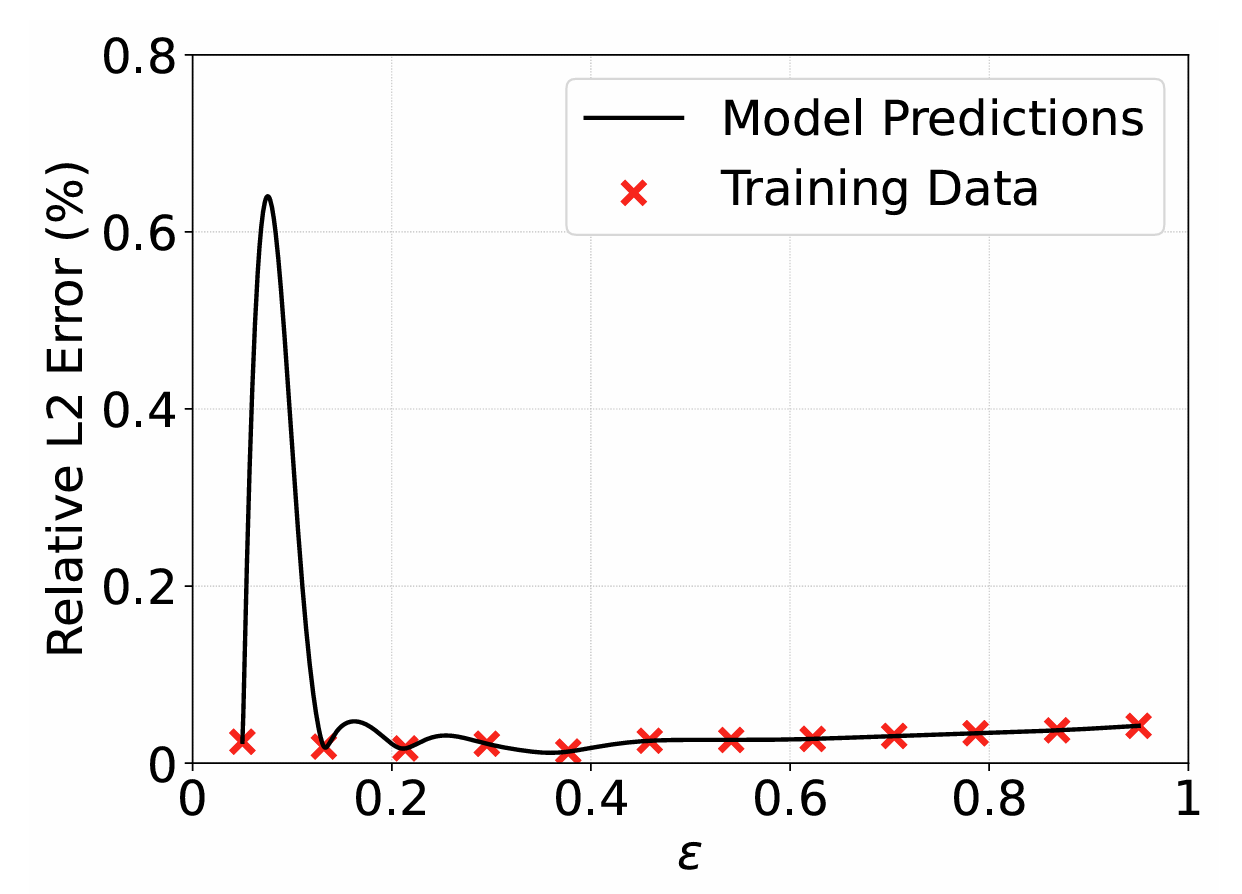}
      \caption{
        Relative $L^2$ error between modeled and analytic solutions as a function of $\varepsilon$, with red \textsc{x} marks indicating training points. 
      }
      \label{fig:1d_ad_test_rel_l2_error}
    \end{minipage}
  \end{subfigure}

  \vspace{1em} 

  \begin{subfigure}[t]{\linewidth}
    \centering
    \begin{minipage}[t][6cm][t]{\linewidth}
    \centering
      \includegraphics[width=0.9\linewidth]{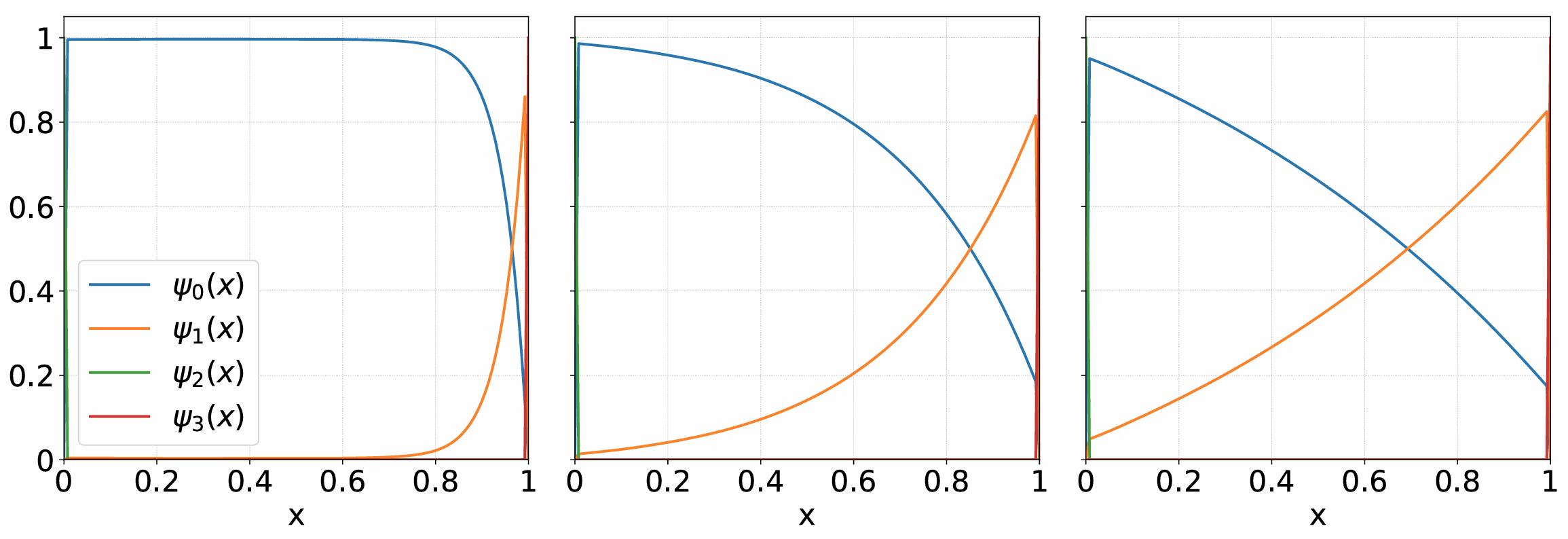}
      \caption{
        Coarse shape functions $\{\psi_i(x)\}$ evaluated on the fine mesh nodes for three values of $\varepsilon$, with reconstructed solution as a dotted line. By conditioning, we obtain a compact basis adapted to the boundary layer. 
      }
      \label{fig:1d_ad_pous_and_solns}
    \end{minipage}
  \end{subfigure}

  \caption{
    Performance of the learned shape function framework on the 1D advection-diffusion problem.
    (a) Example training and validation solutions;
    (b) Test relative $L^2$ error when varying $Z = \varepsilon$;
    (c) Coarse shape functions and reconstructions for representative $\varepsilon$.
  }
  \label{fig:1d_ad_combined}
\end{figure}

\subsection{Extensions to complex geometries and unstructured meshes}\label{exp:bellAD}

We extend the previous case to the steady multi-dimensional advection-diffusion equation:
\begin{equation}
\mathbf{\beta} \cdot \nabla u - \varepsilon \nabla^2 u = 0,
\label{eq:2d_ad}
\end{equation}
where \( \boldsymbol{\beta} \) is a prescribed advection field, and the local Peclet number is defined by \( \mathrm{Pe} = |\boldsymbol{\beta}| L / \varepsilon \) for a characteristic length scale \( L \). To examine performance on complex geometries, we consider a triangular mesh of Philadelphia’s historic Liberty Bell~\cite{Nash2010}. Dirichlet boundary data is imposed via the lifting strategy described in Section~\ref{liftingsection}, with \( u = -1 \) applied on the bell’s crack, \( u = 1 \) at the top handle, and \( u = 0 \) on the remaining boundary segments (see Figure~\ref{fig:bell_mesh}). We set \( \varepsilon = 0.01 \) and \( |\boldsymbol{\beta}| = 1 \), while varying the advection direction as \( \boldsymbol{\beta} = (\cos Z, \sin Z) \).

This parameterization allows us to condition on flow direction. The resulting shape functions (Figure~\ref{fig:bell_shape_functions}) adapt to align with the advective transport field after conditioning. For each direction \( Z \), we solve equation~\eqref{eq:2d_ad} using continuous Galerkin finite elements with \( P_1 \) Lagrange basis functions on the mesh in Figure~\ref{fig:bell_mesh}. The resulting relative L2 error ranges between 0.67\% and 1.49\% (Figure \ref{fig:bell_error}).

\begin{figure}[h]
  \centering

  \begin{subfigure}[t]{0.30\linewidth}
    \centering
    \includegraphics[width=\linewidth]{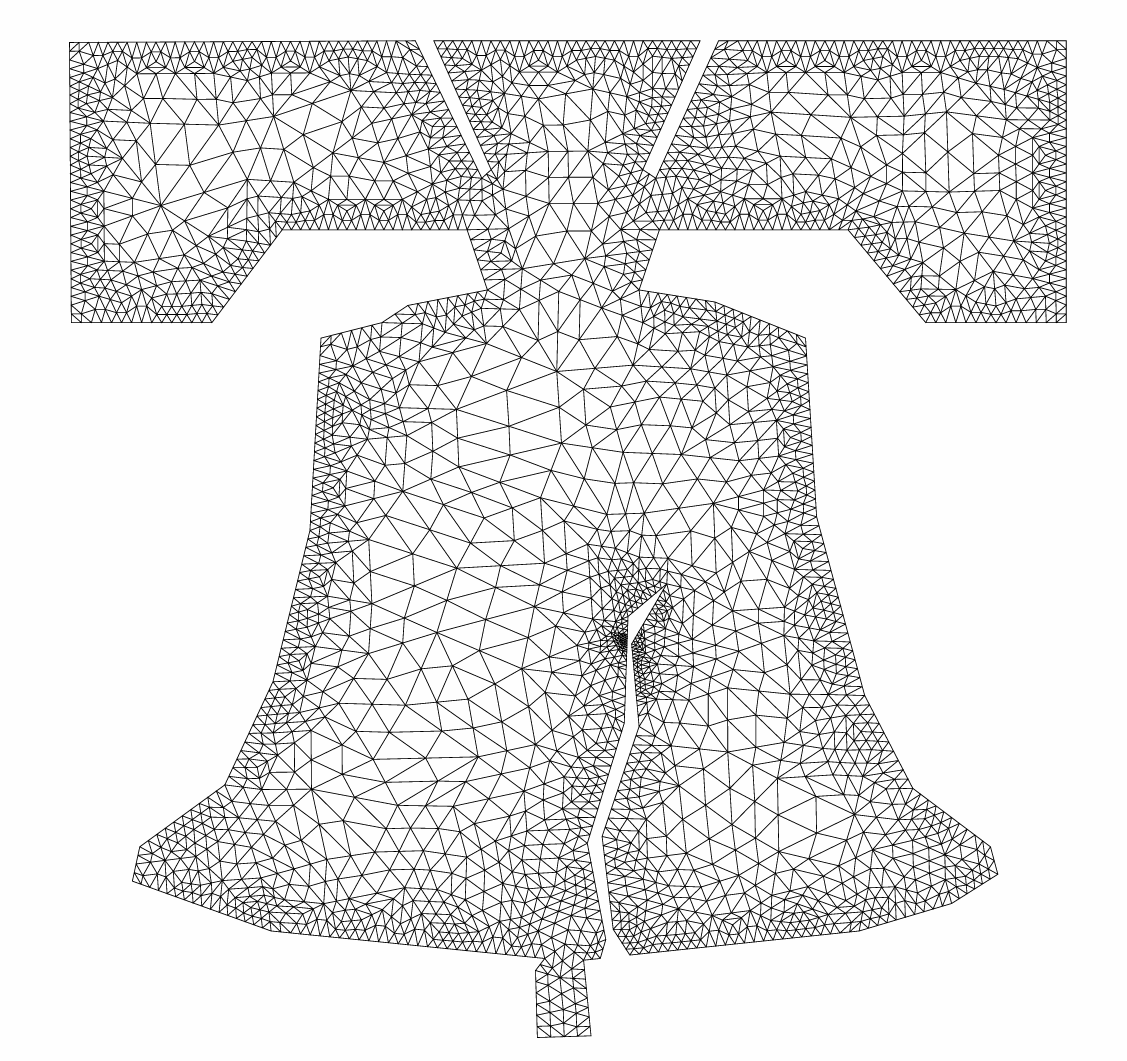}
    \caption{
      Triangular mesh of the Liberty Bell. Nonhomogeneous Dirichlet data is applied to the crack and handle of the bell to highlight the lifting procedure.
    }
    \label{fig:bell_mesh}
  \end{subfigure}
  \hfill
  \begin{subfigure}[t]{0.65\linewidth}
    \centering
    \includegraphics[width=\linewidth]{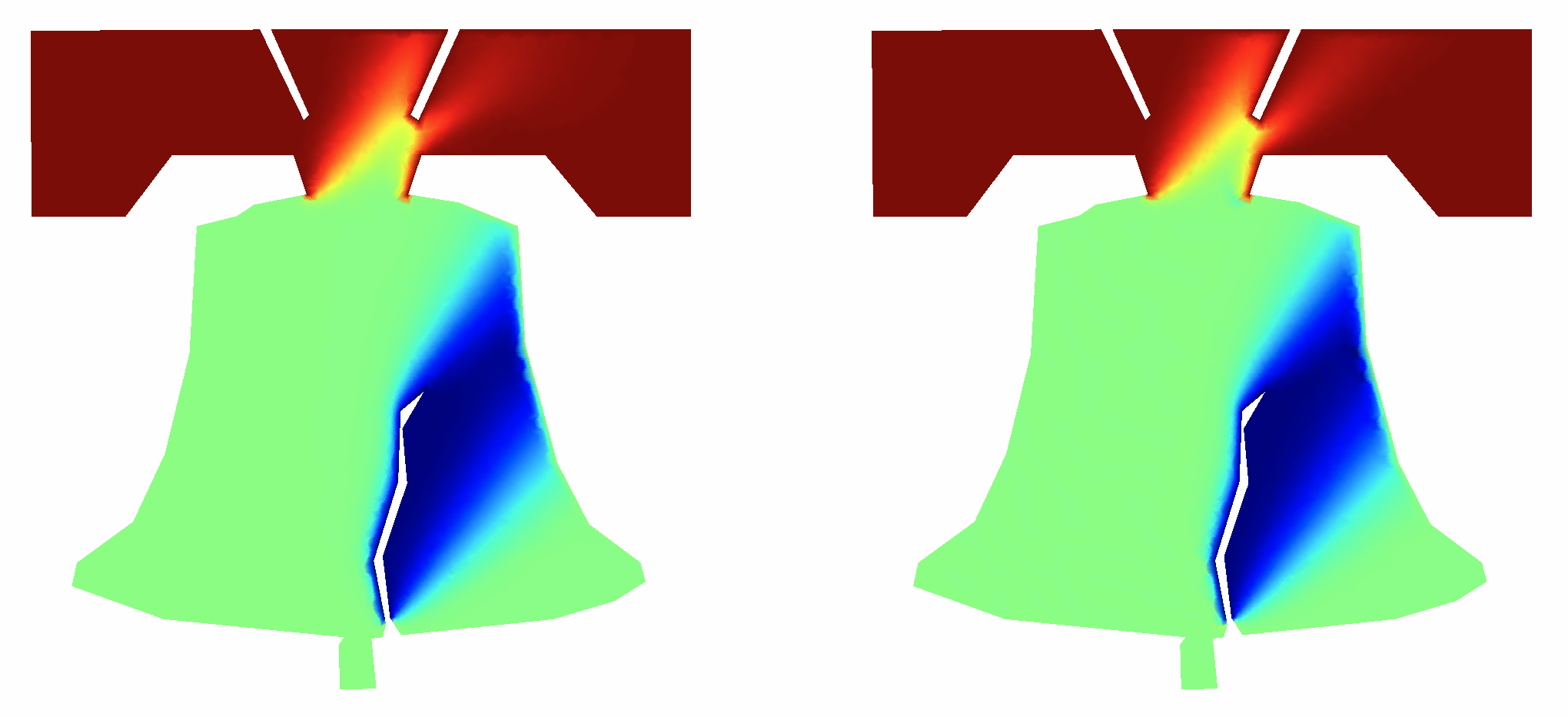}
    \caption{
      Modeled solution (left) compared to the target solution (right) computed using standard FEM for $\Theta = 41^\circ$. This validation example lies between two training examples at $Z = 35^\circ$ and $46^\circ$.
    }
    \label{fig:bell_solns}
  \end{subfigure}

  \vspace{1em} 

  \begin{subfigure}[t]{0.55\linewidth}
    \centering
    \includegraphics[width=\linewidth]{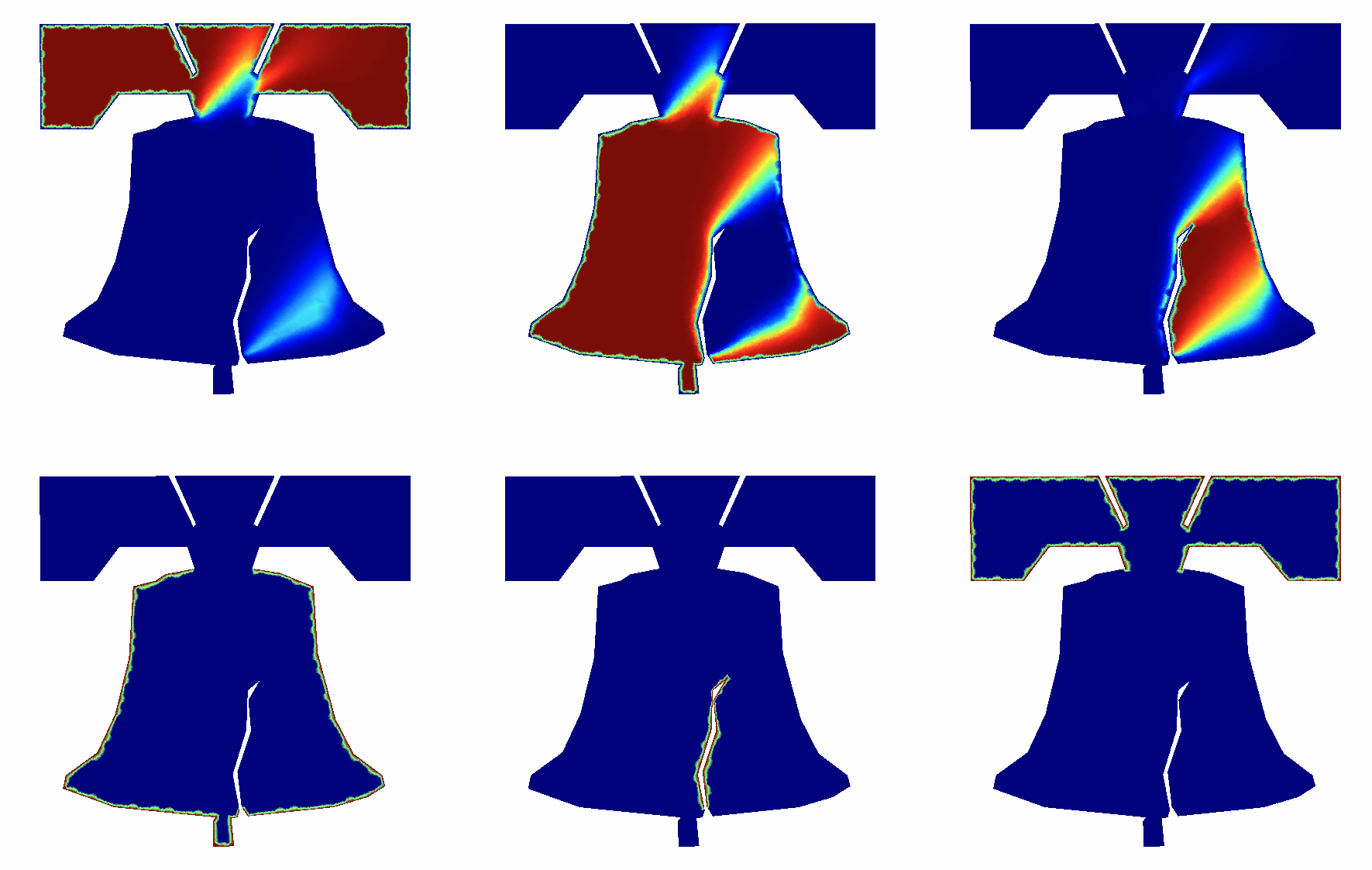}
    \caption{
      Whitney zero-forms for $Z = 41^\circ$. The bottom row shows fixed boundary shape functions used in the Dirichlet lift; the top row shows learned internal partitions.
    }
    \label{fig:bell_shape_functions}
  \end{subfigure}
  \hfill
  \begin{subfigure}[t]{0.4\linewidth}
    \centering
    \includegraphics[width=\linewidth]{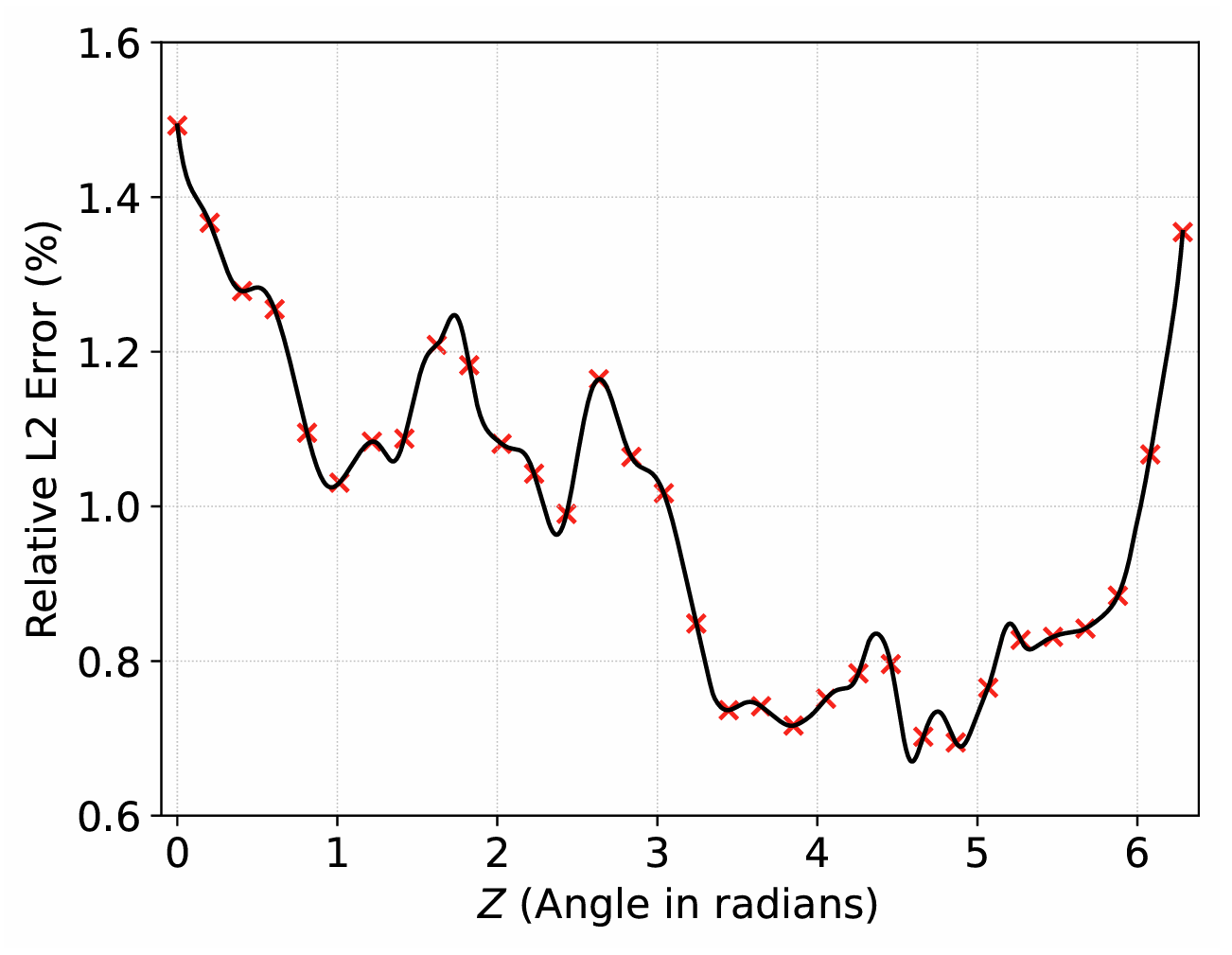}
    \caption{
      Test relative $L^2$ error for $Z\in[0,2 \pi]$. The training set consists of 32 equispaced values for $Z$, denoted by red \textsc{x} marks.
    }
    \label{fig:bell_error}
  \end{subfigure}

  \caption{
    Learned shape function framework on the Liberty Bell geometry.
    (a) Triangular mesh of the Liberty Bell;
    (b) Modeled and target solutions for an intermediate validation angle;
    (c) Learned interior shape functions;
    (d) Test relative $L^2$ error when varying flow direction $Z$;
  }
  \label{fig:bell_combined}
\end{figure}



\subsection{Spacetime shock physics}\label{exp:shocks}

While the previous examples have considered steady-state solutions to scalar-valued, linear conservation laws, we next consider Riemann problems associated with vector-valued, unsteady conservation laws. Recall Euler's equations,

\begin{equation}
    \frac{\partial}{\partial t}
    \begin{bmatrix}
    \rho \\
    \rho v \\
    E
    \end{bmatrix}
    + 
    \frac{\partial}{\partial x}
    \begin{bmatrix}
    \rho v \\
    \rho v^2 + p \\
    v(E + p)
    \end{bmatrix}
    = 0
    \label{eq:euler_equations}
\end{equation}
with associated conserved density fields: mass density $\rho$, momentum density $M = \rho v$, for a velocity $v$, and energy density $E$. For simplicity, we consider the ideal gas equation of state written in terms of conservative variables
\begin{equation}
    p = (\gamma - 1) \left(E - \frac{M^2}{2\rho}\right).
\end{equation}
This closure is parameterized by the adiabatic index $\gamma$, which we will consider as our conditioning variable to demonstrate in a simple setting how one could condition over material parameters or other parametric dependence on constitutive relationships. For the domain $[-L/2, L/2] \times [0, 1]$ for $L = 7$ in $(x, t)$ coordinates, we consider the Sod shock tube initial conditions \citep{SOD19781}

\begin{align}
\rho(x) &= 
\begin{cases}
3, & \text{if } x < 0 \\
1, & \text{if } x \geq 0
\end{cases} \notag\\
M(x) &= 0, \notag\\
E(x) &= 
\begin{cases}
\frac{3}{\gamma - 1}, & \text{if } x < 0 \\
\frac{1}{\gamma - 1}, & \text{if } x \geq 0
\end{cases}
\label{eq:sod_ics}
\end{align}

As an underlying fine mesh we consider a $128 \times 64$ Cartesian grid of continuous P1 triangular elements, and generate as data the solutions for 21 equispaced $\gamma \in [2,7]$.

To treat the unsteady physics in this problem, the machinery introduced in this paper could in principle be embedded in a neural time integrator like NeuralODE \cite{chen2018neural}. Instead we pose our problem in spacetime following our previous work \cite{patel2022thermodynamically}. Define the \textit{spacetime differential} by concatenating the temporal partial derivative and the spatial gradient, $\widetilde{\nabla} = [\partial_t; \nabla]$. With this modification, a given unsteady conservation law
\begin{equation}
    \partial_t u + \nabla \cdot F(u) = 0
\end{equation}
may be rewritten as
\begin{equation}
    \widetilde\nabla \cdot \widetilde{F(u)} = 0,
\end{equation}
where we define the augmented spacetime flux $\widetilde{F(u)} = [u; F(u)]$, so that we obtain a steady-state problem. To perform the lift, we partition the spacetime boundary \( \partial \Omega \) as \( \partial\Omega= \Gamma_1 \cup \Gamma_2 \cup \Gamma_3 \), where:
\begin{itemize}
\item \( \Gamma_1 = \{(x,t) : x < 0,\ t= 0\ \cup x=-L/2,\ t \neq 1\} \) imposes the left initial condition and wall boundary;
\item \( \Gamma_2 = \{(x,t) : x \geq 0,\ t = 0\ \cup x=L/2,\ t \neq 1\} \) imposes the right initial condition and wall boundary;
\item \( \Gamma_3 = \{(x,t) : t = 1\} \) prescribes a homogeneous Neumann outflow condition.
\end{itemize}
These boundary partitions define the Dirichlet lift along the initial condition and wall boundaries (bottom, left, and right) and a homogeneous Neumann condition at the "outflow" (top). Results in Figure~\ref{fig:riemann_dir_combined} demonstrate the performance of the method, including representative solutions (Figure~\ref{fig:riemann_dir_training_solns}), the learned partition of unity functions (Figure~\ref{fig:riemann_dir_pou_plots}), and the relative \( L^2 \) error across validation cases (Figure~\ref{fig:riemann_dir_rel_l2_error}). Notably, the learned shape functions exhibit sharp localization near normal shocks and smooth transitions across rarefactions, yielding non-oscillatory reconstructions even in the absence of explicit shock limiters. Whereas classical high-resolution finite element or finite volume methods typically require slope or flux limiters to suppress Gibbs phenomena near discontinuities~\cite{sweby1984high,cockburn1989tvb,kuzmin2012flux}, our model implicitly learns a flux representation that stabilizes the solution and suppresses spurious oscillations in a data-driven manner. This is particularly significant given that the partitions are constructed from \emph{continuous} shape functions, a setting where advanced techniques such as flux-corrected transport (FCT) are often necessary to enforce boundedness and monotonicity~\cite{kuzmin2012flux}. The resulting solutions are efficient enough that the learned model could be e.g. embedded in a hydrodynamics code as a data-driven flux reconstruction.

\begin{figure}[t]
  \centering

  \begin{subfigure}[t]{0.9\linewidth}
    \centering
    \includegraphics[width=\linewidth]{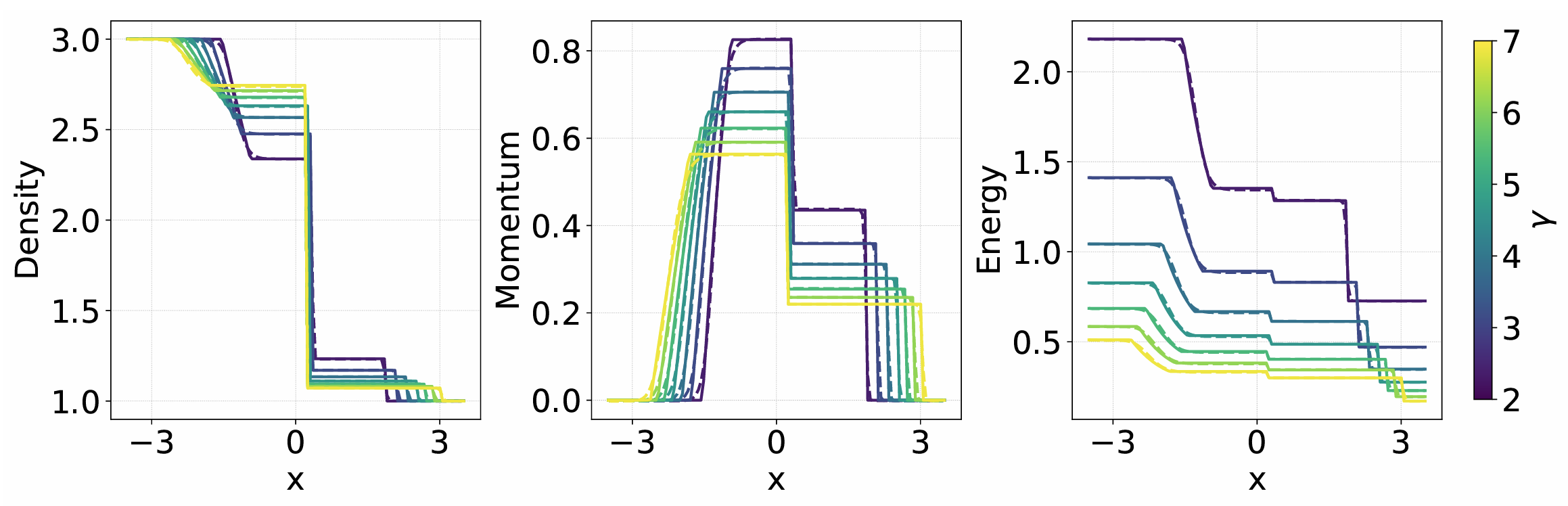}

    \caption{
      True (dashed) and modeled (solid) solutions for the Sod shock tube at $t = 1$, for representative validation values of $\gamma$. Solutions are localized and non-oscillatory in the vicinity of shocks.
    }
    \label{fig:riemann_dir_training_solns}
  \end{subfigure}

  \begin{subfigure}[t]{0.9\linewidth}
    \centering
    \includegraphics[width=\linewidth]{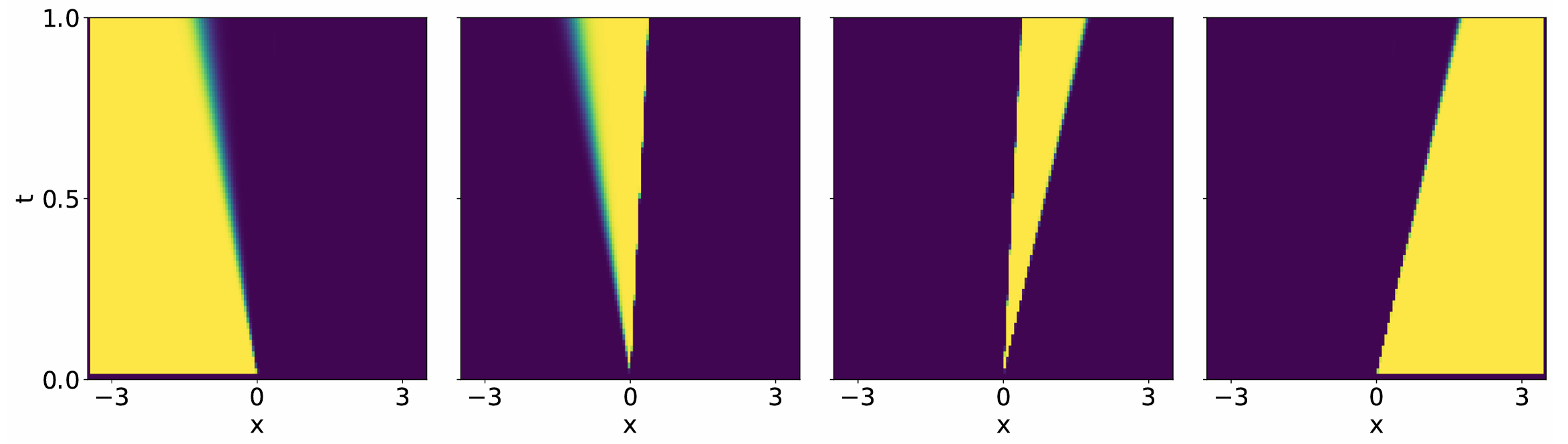}
    \caption{
      Learned internal spacetime partitions comprising the coarse shape function basis, shown for $\gamma = 2$. Remarkably, the learned functions concentrate near normal shocks, sharply capturing discontinuities, while exhibiting smooth, diffuse transitions across rarefaction waves.
    }
    \label{fig:riemann_dir_pou_plots}
  \end{subfigure}

  \begin{subfigure}[t]{0.4\linewidth}
    \centering
    \includegraphics[width=\linewidth]{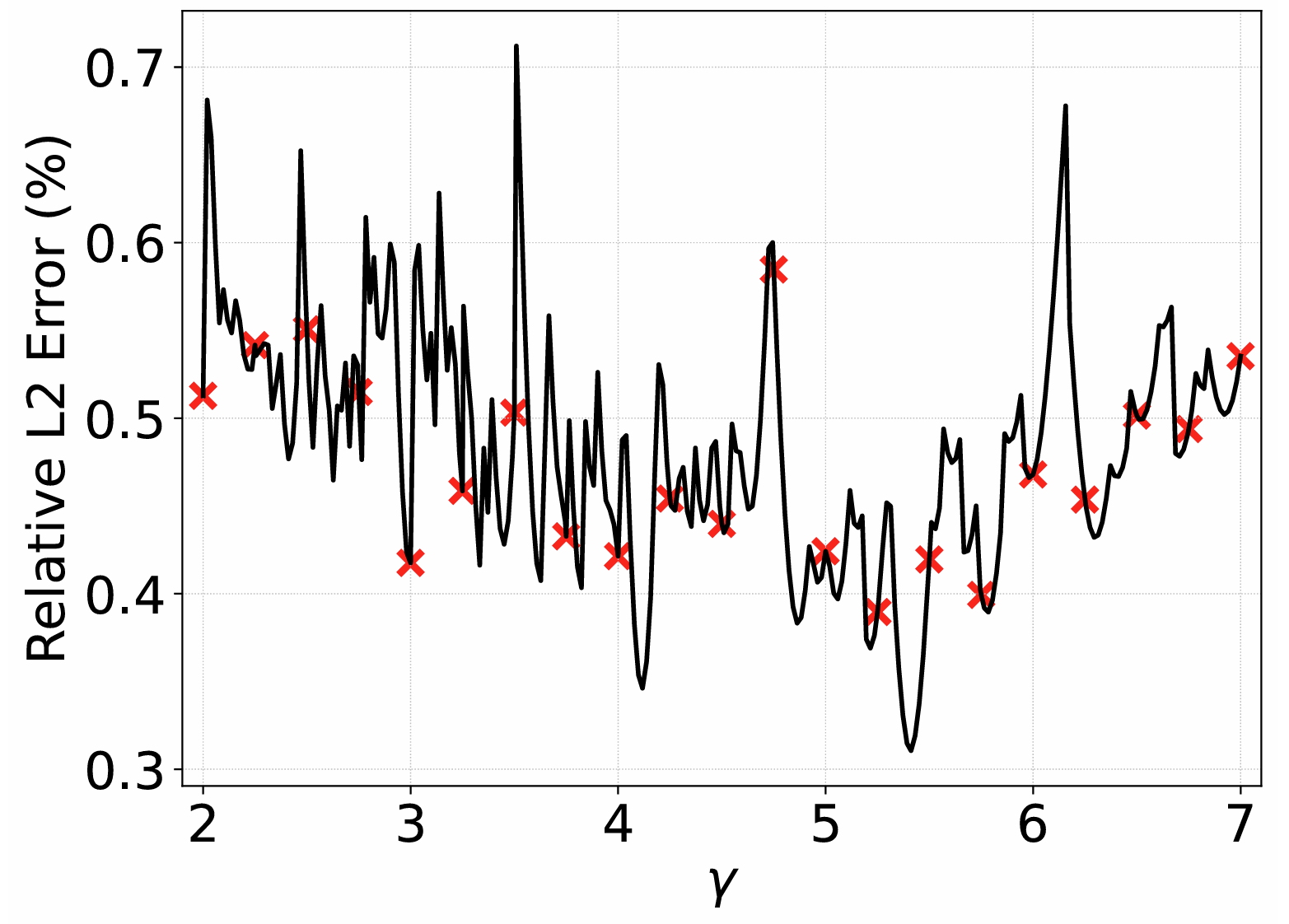}
    \caption{
      Relative $L^2$ error in each scalar field across the training $\gamma$ range. Red \textsc{x} marks indicate training points.
    }
    \label{fig:riemann_dir_rel_l2_error}
  \end{subfigure}
  \hfill
  \begin{subfigure}[t]{0.5\linewidth}
    \centering
    \includegraphics[width=\linewidth]{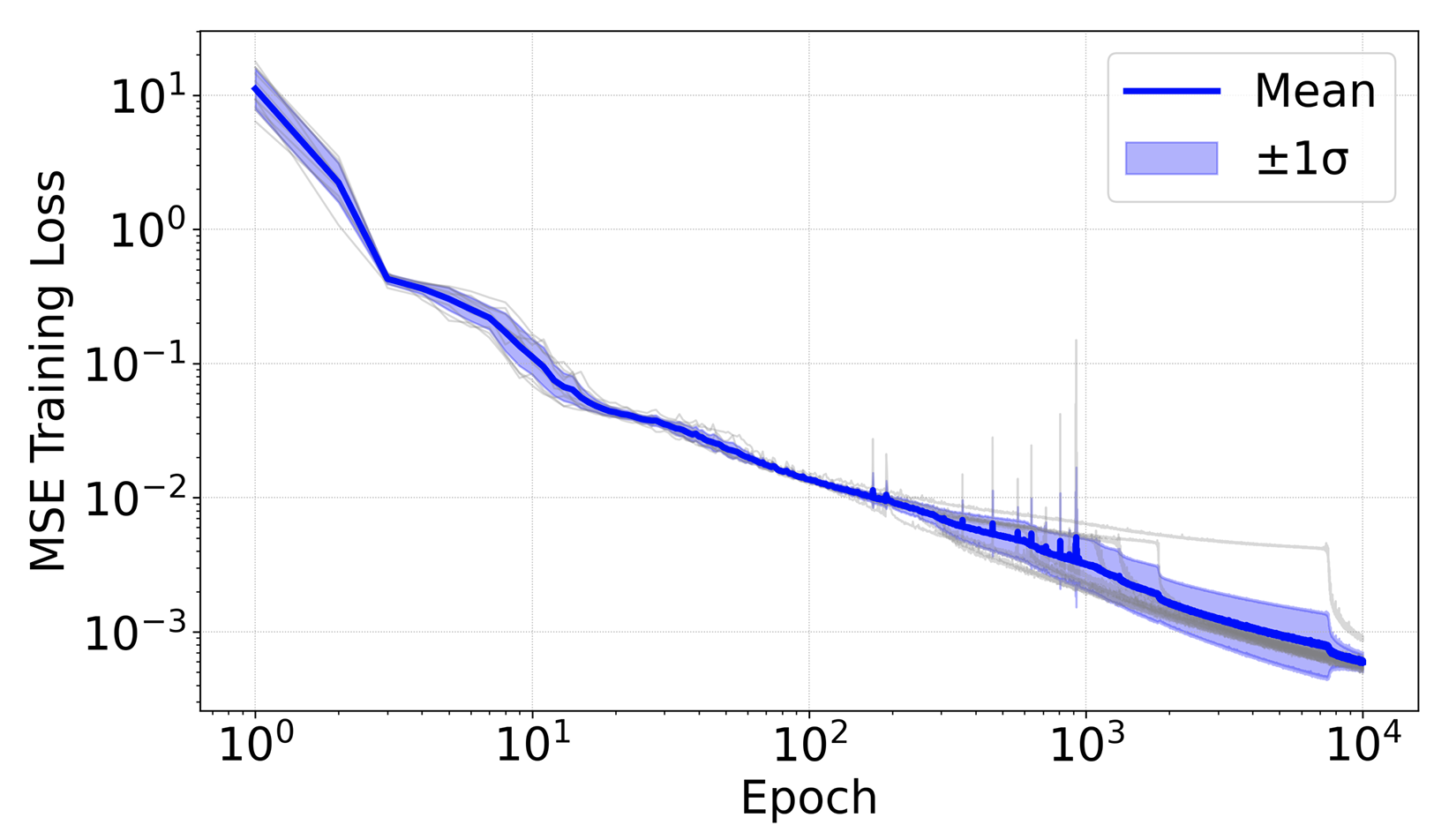}
    \caption{
        Convergence of training loss over an ensemble of ten initializations, highlighting stability of training.
    }
    \label{fig:loss_history_confidence}
  \end{subfigure}

  \caption{
    Performance of the learned model for the Sod shock tube:
    (a) Representative validation solutions; (b) Learned spacetime partitions; (c) Relative $L^2$ error across $\gamma$; (d) Loss history confidence.
  }
  \label{fig:riemann_dir_combined}
\end{figure}

\clearpage

\subsection{Charge in a Conducting Shell}\label{exp:conducting shell}

In the final pedagogical example, we consider the electrostatics of a point charge enclosed in a conducting shell on the unit disk. The electric potential \( u \) satisfies the Poisson equation
\begin{equation}
\nabla^2 u = - \delta(\mathbf{x} - \mathbf{x}_Q),
\label{eq:poisson}
\end{equation}
where \( \delta \) denotes the Dirac delta distribution and \( \mathbf{x}_Q \) is the location of the point charge. We impose a homogeneous Dirichlet boundary condition on the shell, enforced via lift. This problem is notable for the reduced regularity of the solution: due to the singular source term, \( u \notin H_0^1(\Omega) \). Applying the divergence theorem yields the conservation law
\begin{equation}
    \int_{\partial B(0,1)} \nabla u \cdot d\mathbf{A} = -1.
\end{equation}
This example thus serves to verify both global conservation and the learned model’s ability to localize and represent singular behavior. We consider an infinite data setting: at each epoch of training we condition on the charge location by taking $\mathbf{x}_Q = Z$ for $Z$ sampled from a uniform distribution on the disk and solve the corresponding finite element problem with a continuous Galerkin nodal-$P1$ code on a 1550 node triangular mesh, which also serves as the fine mesh for our learning problem.

The results of training can be found in Figure \ref{fig:conducting_sphere_combined}. Due to the singularity, the learned solution can be poor in the vicinity of the charge. Notably, we demonstrate in Figure \ref{fig:conducting_sphere_errors} that the electric flux is predicted to machine precision independent of the quality of reconstruction as a consequence of global conservation. Interestingly, we also observe that the process localizes a partition to the singularity, using the remaining partitions to reconstruct the far field (Figure \ref{fig:conducting_sphere_pous}).

\begin{figure}[t]
  \centering

  \begin{subfigure}[t]{0.5\linewidth}
    \centering
    \begin{minipage}[t][6cm][t]{\linewidth}
      \includegraphics[width=\linewidth]{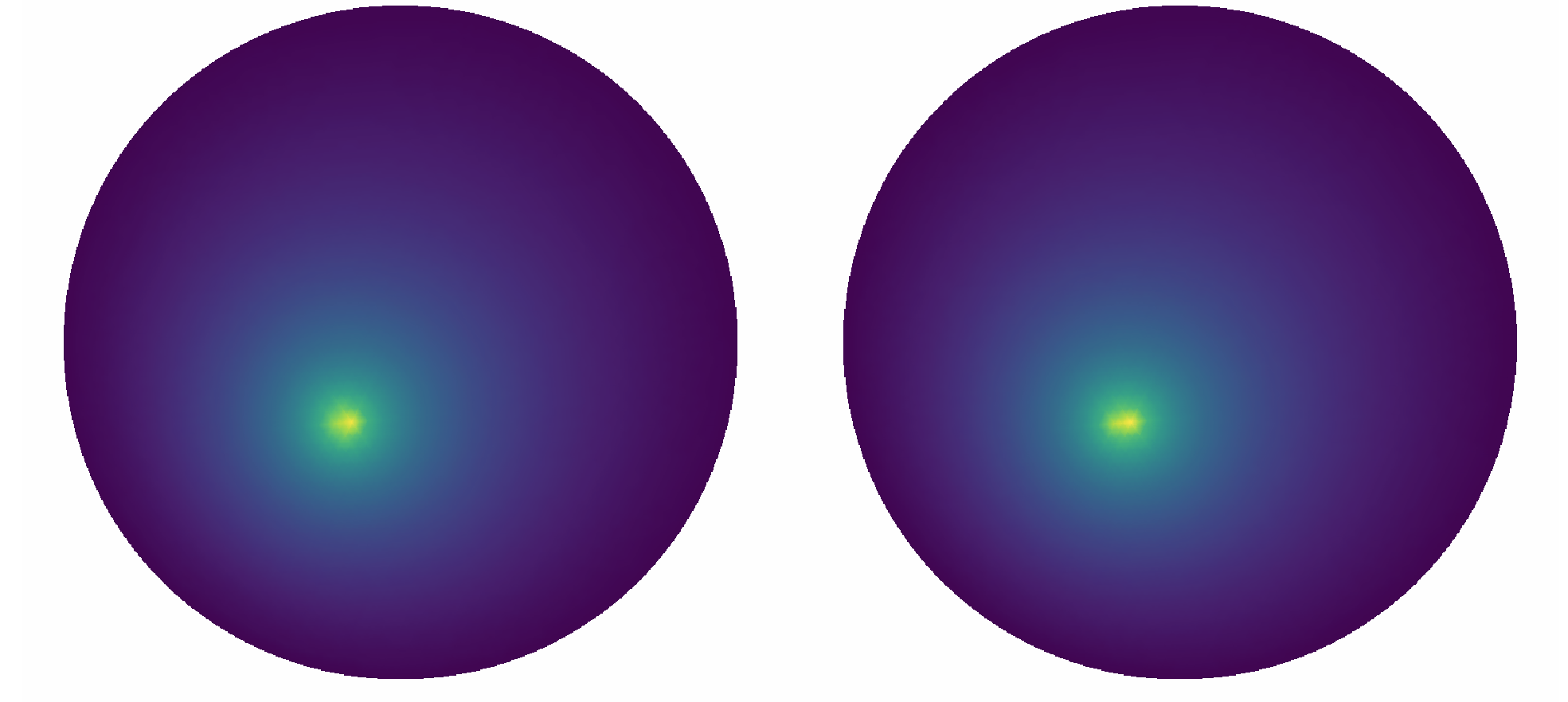}
      \caption{
        Representative solution (left) and target FEM solution (right) for a point charge in a unit disk.
      }
      \label{fig:conducting_sphere_soln_comparison}
    \end{minipage}
  \end{subfigure}
  \hfill
  \begin{subfigure}[t]{0.45\linewidth}
    \centering
    \begin{minipage}[t][6cm][t]{\linewidth}
      \includegraphics[width=\linewidth]{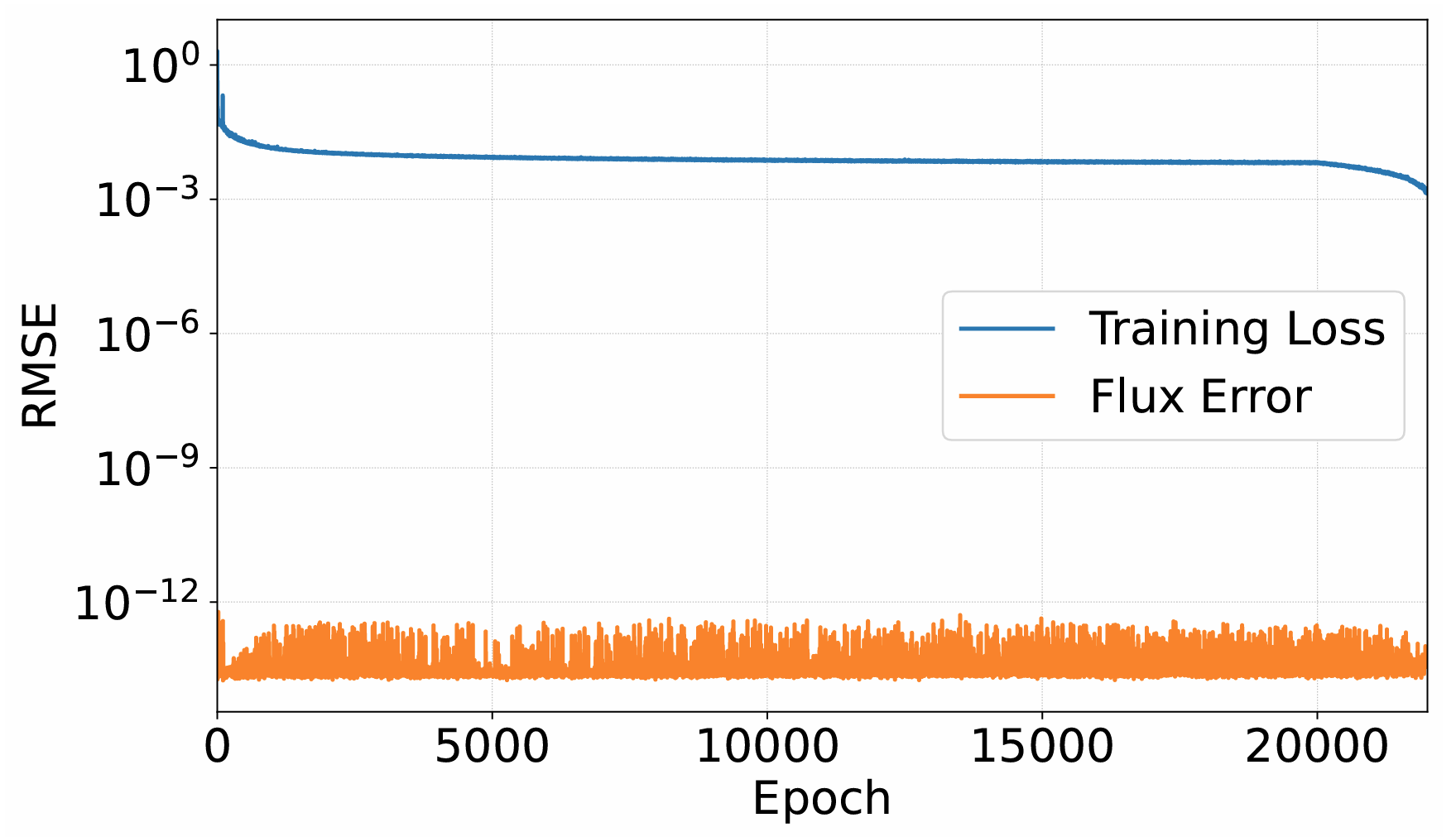}
      \caption{
        RMSE error and flux conservation error. The flux error measures the difference between point source and electric flux out of domain, demonstrating conservation to machine precision.
      }
      \label{fig:conducting_sphere_errors}
    \end{minipage}
  \end{subfigure}

  \vspace{1em}

  \begin{subfigure}[t]{\linewidth}
    \centering
    \includegraphics[width=0.9\linewidth]{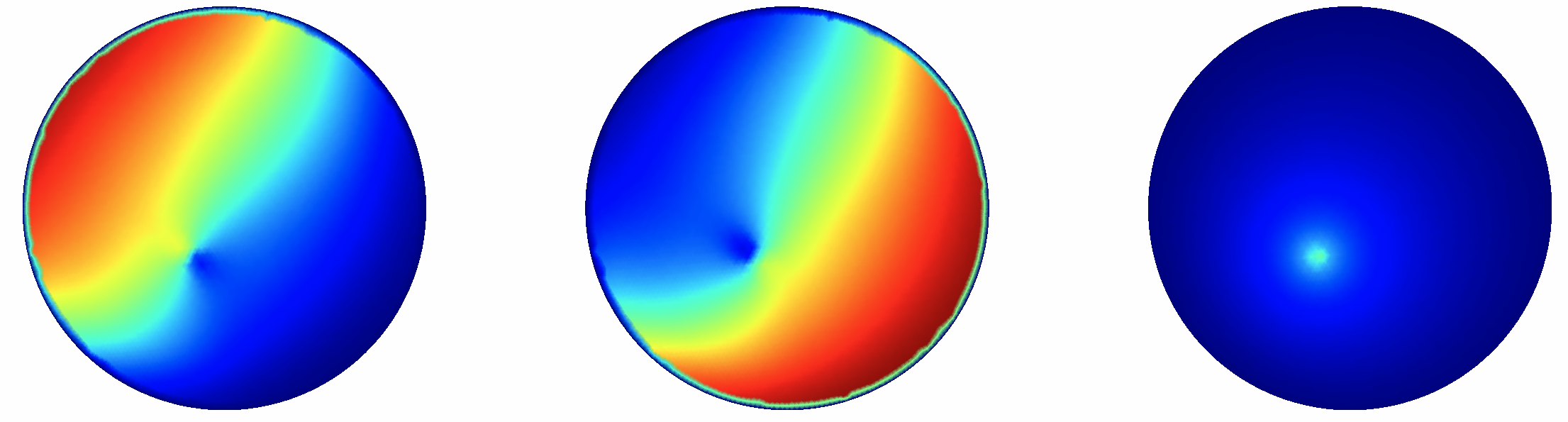}
    \caption{
      The three representative shape functions conditioned on a charge in the lower left quadrant. One shape function is localized to the singularity while the remainder reconstruct the far field. This adaptation to the singularity tracks the point charge as the location $Z$ is varied.
    }
    \label{fig:conducting_sphere_pous}
  \end{subfigure}

  \caption{
    Performance of the learned shape function framework for the charge-in-conducting-shell problem:
    (a) Solution comparison; (b) Batchwise RMS and conservation errors; (c) Learned shape functions.
  }
  \label{fig:conducting_sphere_combined}
\end{figure}


Figure \ref{fig:conducting_sphere_soln_comparison} presents a comparison between the modeled solution and the target FEM solution for a unit charge located in the lower left quadrant of the unit circle. The median relative L2 error over 640 random samplings of the charge location is 2.05\%. In Figure \ref{fig:conducting_sphere_pous}, we show the three learned coarse shape functions for the same example shown in Figure \ref{fig:conducting_sphere_soln_comparison}. Note that the model learns to use one of the shape functions to ``punch out'' a small region centered on the charge's location.

\subsection{Digital twin of a battery rack undergoing thermal runaway}\label{exp:digitaltwin}

Finally, we construct a digital twin for thermal management in a rack of lithium-ion batteries. Thermal runaway, initiated by mesoscopic failure modes, triggers cascading exothermic reactions that cause a self-amplifying temperature rise, often culminating in catastrophic failure such as fire or structural damage~\cite{Wang2012,Finegan2017}. The timescale of this process is strongly influenced by convective heat transfer, which can be modulated through active flow control.

Fully resolving the coupled electromechanical and thermal physics would require microstructure-resolving models for battery internals and high-fidelity DNS or LES for external flow—both computationally intensive. As a result, most studies rely on restrictive scale separation assumptions that decouple battery and flow physics, potentially missing critical interactions \cite{Zhang2021}. A digital twin capable of capturing this coupled multiscale system would enable predictive modeling of material–HVAC interactions and support control strategies that mitigate thermal runaway by activating cooling at the earliest signs of instability. In earlier work, we constructed a digital twin of the transport in the battery using as-built tomography data \cite{actor2024data}. Here, we construct a twin of the surrounding environment, using conditioning as a simple mechanism to couple the fluid response to the surface temperature of a battery module undergoing thermal runaway colocated with other nominally performing modules.


We use a battery rack with a simplified, generic geometry composed of six vertically stacked battery modules, represented as rectangular prisms \cite{Meehan2025}. The modules are enclosed in the rack with openings at the bottom left and top right with an atmospheric boundary condition. To mimic thermal runaway, a Dirichlet condition is imposed on the temperature of the second module from the bottom and Low-Mach number LES model is solved for density, momentum, and energy to determine the resulting heat transfer through the rack. For natural convection the transition to turbulence is characterized by the Grashoff number, and so we generate a database of solutions using the Sierra/Fuego solver ~\cite{FUEGO} developed at Sandia National Laboratories, varying (1) viscosity ($\mu$) and (2) the temperature difference ($\Delta T$) of the runaway module. We condition on the working fluid and the battery temperature by taking $Z = [\mu, \Delta T]$ with the objective of predicting the transition to turbulence across a range of Grashoff numbers.


In natural convection, the transition to turbulence is governed by the Grashof number,
\[
\mathrm{Gr} = \frac{g\,\beta\,\Delta T\,L^3}{\nu^2},
\]
where \(g\) is gravity, \(\beta\) is the thermal expansion coefficient, \(\Delta T\) is the driving temperature difference, \(L\) is the characteristic length, and \(\nu\) is the kinematic viscosity~\cite{Incropera2007}. We generate a dataset by varying \(\mu\) and \(\Delta T\) and condition on \(Z = [\mu, \Delta T]\). For canonical vertical-wall configurations, the transition to turbulence occurs near \(\mathrm{Gr} \approx 10^9\)~\cite{Incropera2007}, and we expect to extract a model which predicts a transition in a similar range of $\mathrm{Gr}$. 

The dataset consists of 25 simulations consisting of a $5\times5$ Cartesian grid of $\mu \in \left[1.8 \cdot 10^{-4}-1.0 \cdot 10 ^{-3}\right]$ g/cm-s and $\Delta T\in \left[2.4- 75\right]$ K. Each simulation averaged around 86,400 CPU-hours to complete using a multi-CPU cluster, and so the relative sparsity of a $25$ sample dataset in this simple configuration still represents a major data curation effort, highlighting the need for a sample-efficient approach. Each simulation consists of a time-series prescribing an LES solution of density, velocity, temperature, enthalpy, pressure, momentum, and the turbulent kinetic energy (TKE).

We build a real-time simulator by developing a data-driven Reynolds‑Averaged Navier–Stokes (RANS) solver. In RANS, an instantaneous flow variable \(f(\mathbf{x},t)\) is decomposed into a mean and a fluctuating component,
\[
f(\mathbf{x},t) = \overline{f(\mathbf{x},t)} + f'(\mathbf{x},t),
\]
with the ensemble or time average satisfying \(\overline{f'} = 0\)~\cite{durbin2011statistical}. This leads to the RANS equations, in which the mean flow is governed by additional Reynolds-stress terms that must be closed via turbulence models. Although RANS is widely used in industrial CFD, the closures generally provide \textit{qualitative insight}, are empirically tuned and often lack quantitative accuracy without calibration to experimental data~\cite{Duraisamy2018} and generally struggle to predict the transition to turbulence \cite{LIVESCU2008,Schwarzkopf2016}. In contrast to classical approaches that postulate closure equations for turbulent quantities in an ad hoc manner, we posit data-driven conservation laws for the Reynolds-averaged density, momentum, pressure, enthalpy, and turbulent kinetic energy fields. From the high-fidelity LES data, we construct steady-state RANS profiles and reverse engineer a conditioned neural weak form (CNWF) model whose solution yields a self-consistent RANS prediction. Further details may be found in \ref{app:nn_details}.

Representative results are shown in Figure~\ref{fig:battery_val_and_pou}. We observe strong agreement with training data (Figure~\ref{fig:battery_val_comparison}), with partitions that clearly capture wake and plume features in various fields. Quantitative assessment (Figure~\ref{fig:TKE_combined}) shows relative errors below 10\% for all hydrodynamic variables, along with accurate turbulence transition prediction. Such accuracy is notable for RANS models, as traditional closures typically achieve only within a factor of two agreement in many flow regimes~\cite{NASA2020}. Our model performs inference in under one second, providing a speedup of approximately \(3.11\times10^8\) compared to the 86,400 CPU-hours required to generate the training data.

\begin{figure}[t]
  \centering

  \begin{subfigure}[t]{\linewidth}
    \centering
    \includegraphics[width=0.95\linewidth]{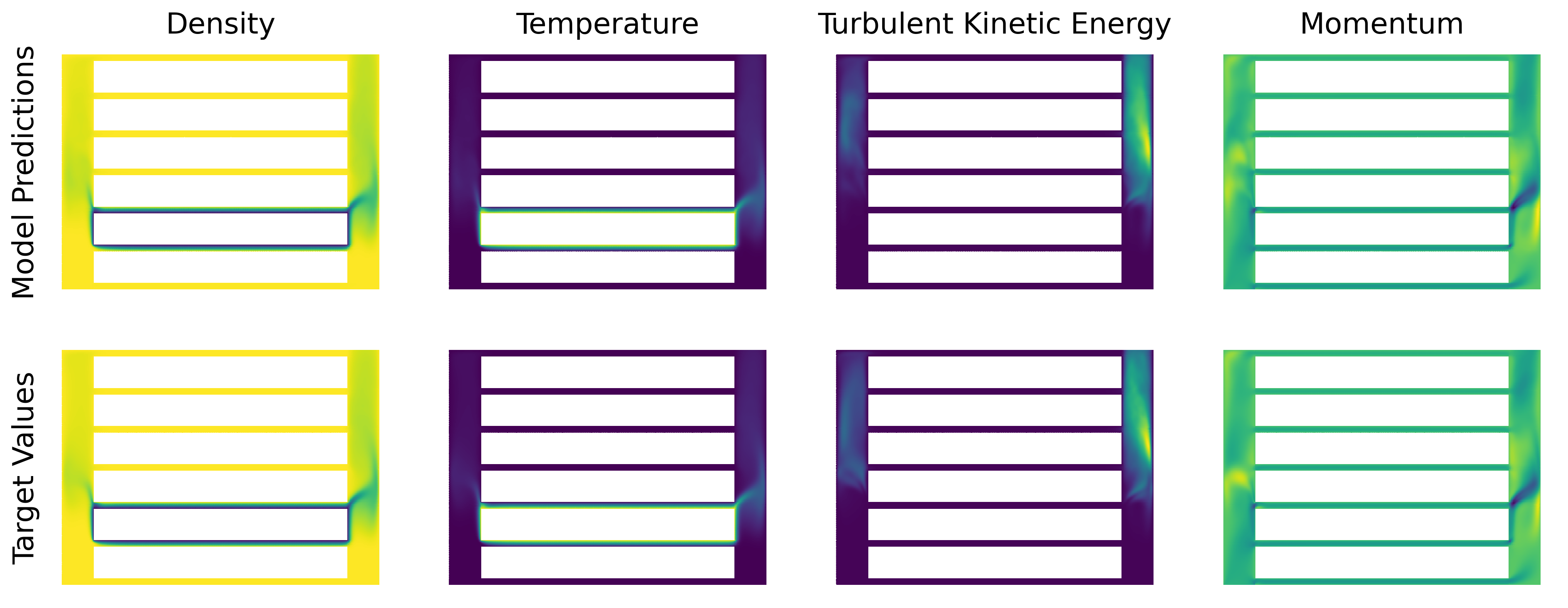}
    \caption{The model predictions (left) and the targets (right), for all scalar fields, for the held-out validation example with $\Delta T = 31.6$~K, $\mu = 6.7 \times 10^4$~g/cm-s. Fields are scaled from minimum (blue) to maximum (yellow) values.}
    \label{fig:battery_val_comparison}
  \end{subfigure}

  \vspace{1em}

  \begin{subfigure}[t]{\linewidth}
    \centering
    \includegraphics[width=0.8\linewidth]{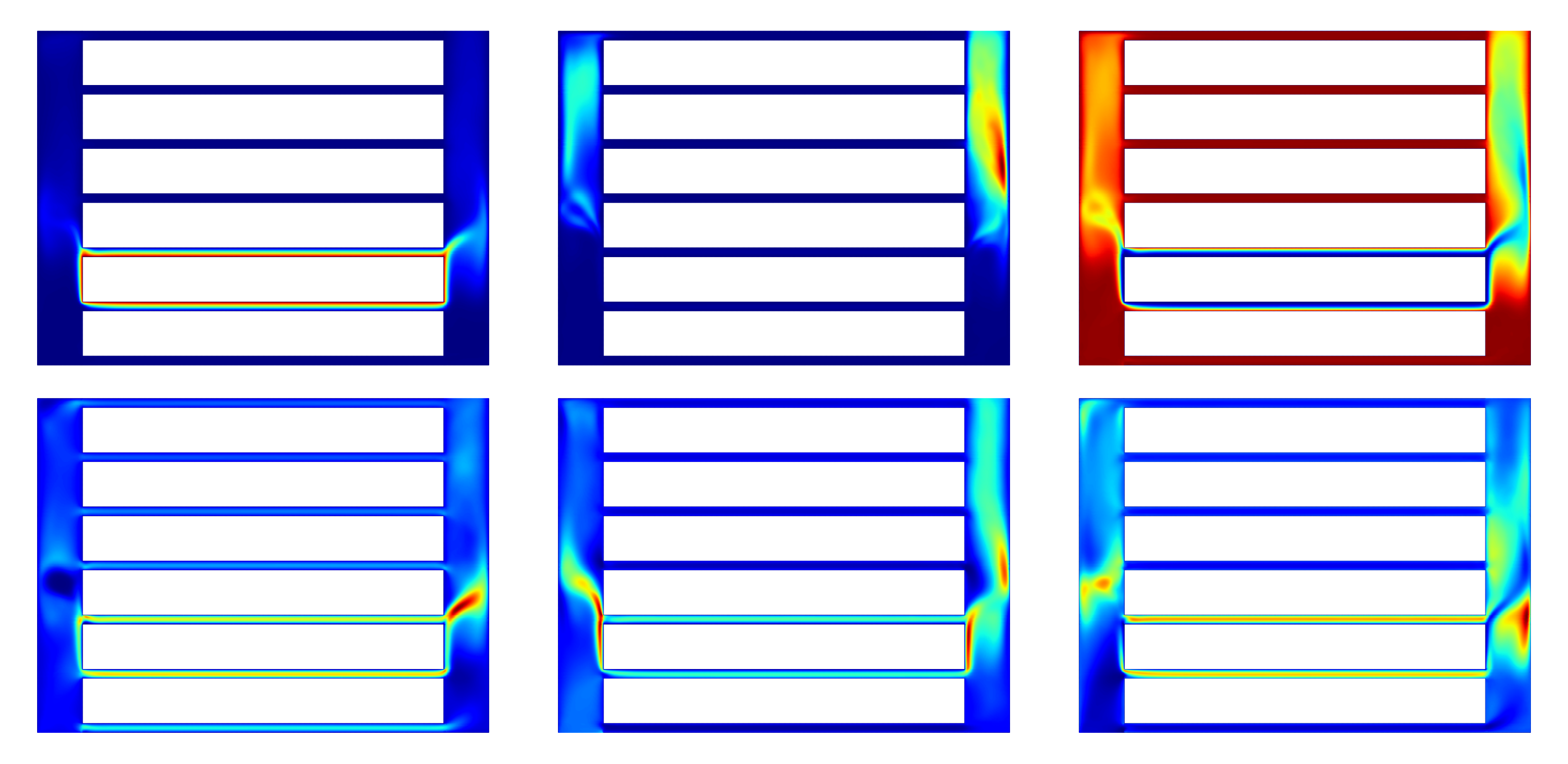}
    \caption{The six learned shape functions for the training example shown in Figure~\ref{fig:battery_val_comparison}. All but one shape function are used to capture different aspects of the plume from the hot battery module (the second from the bottom of the rack). Shape functions enforcing Dirichlet boundary conditions on battery walls and outer casing are not shown.}
    \label{fig:battery_pou_shapes}
  \end{subfigure}

  \caption{Performance of the learned model for the battery thermal runaway problem. (a) Validation solutions. (b) Learned shape functions conditioned on $\Delta T$ and $\mu$.}
  \label{fig:battery_val_and_pou}
\end{figure}

\begin{figure}[t]
  \centering

  \begin{subfigure}[t]{0.5\linewidth}
    \centering
    \begin{minipage}[t][8.5cm][t]{\linewidth}
      \includegraphics[width=\linewidth]{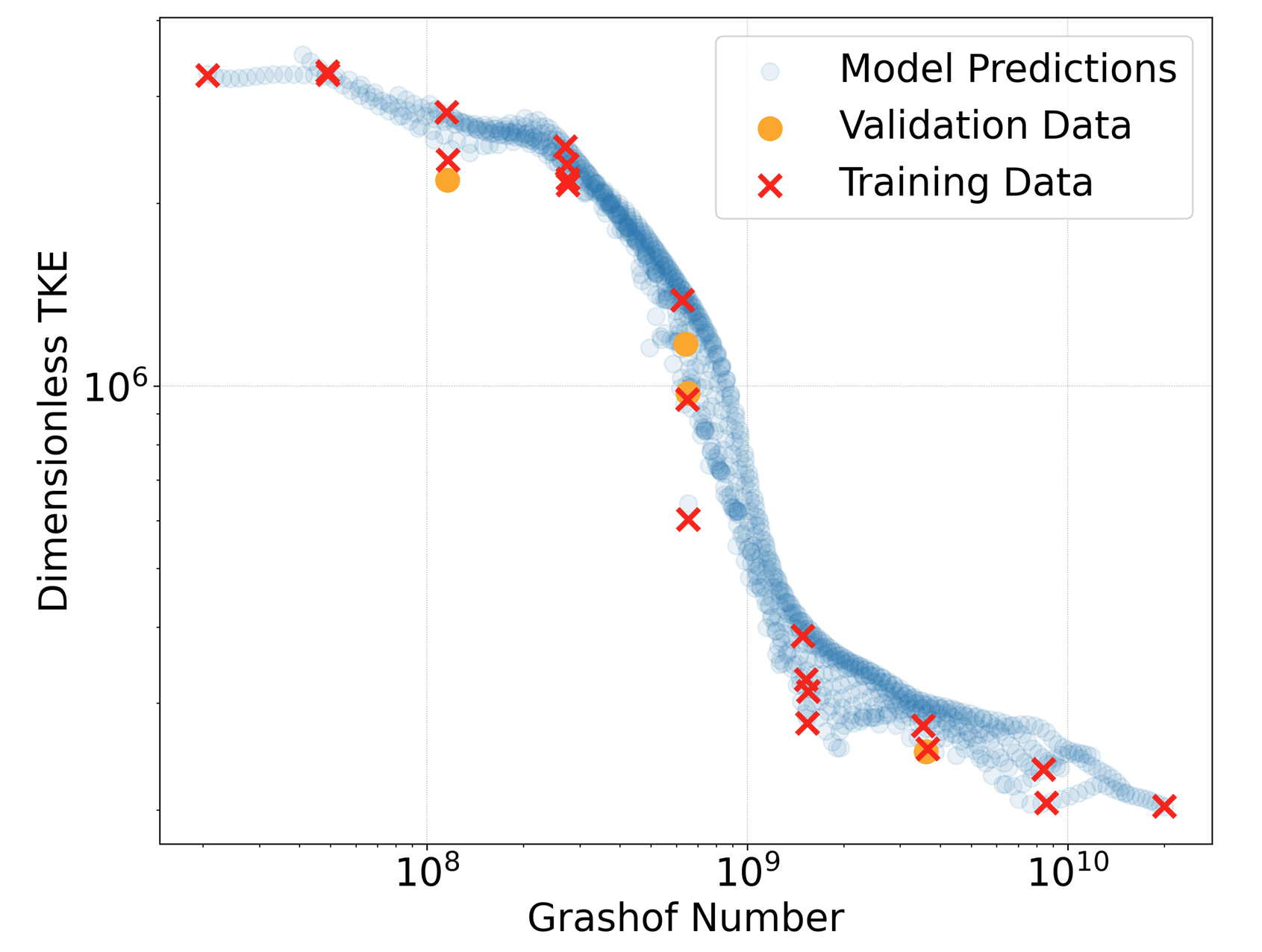}
      \caption{
        Predicted dimensionless TKE integrated over domain. Despite training on only 20 LES, we capture the transition to turbulence at $Gr = 10^9$.
      }
      \label{fig:grashof_tke_scatter}
    \end{minipage}
  \end{subfigure}
  \hfill
  \begin{subfigure}[t]{0.4\linewidth}
    \centering
    \begin{minipage}[t][8.5cm][t]{\linewidth}
      \includegraphics[width=\linewidth]{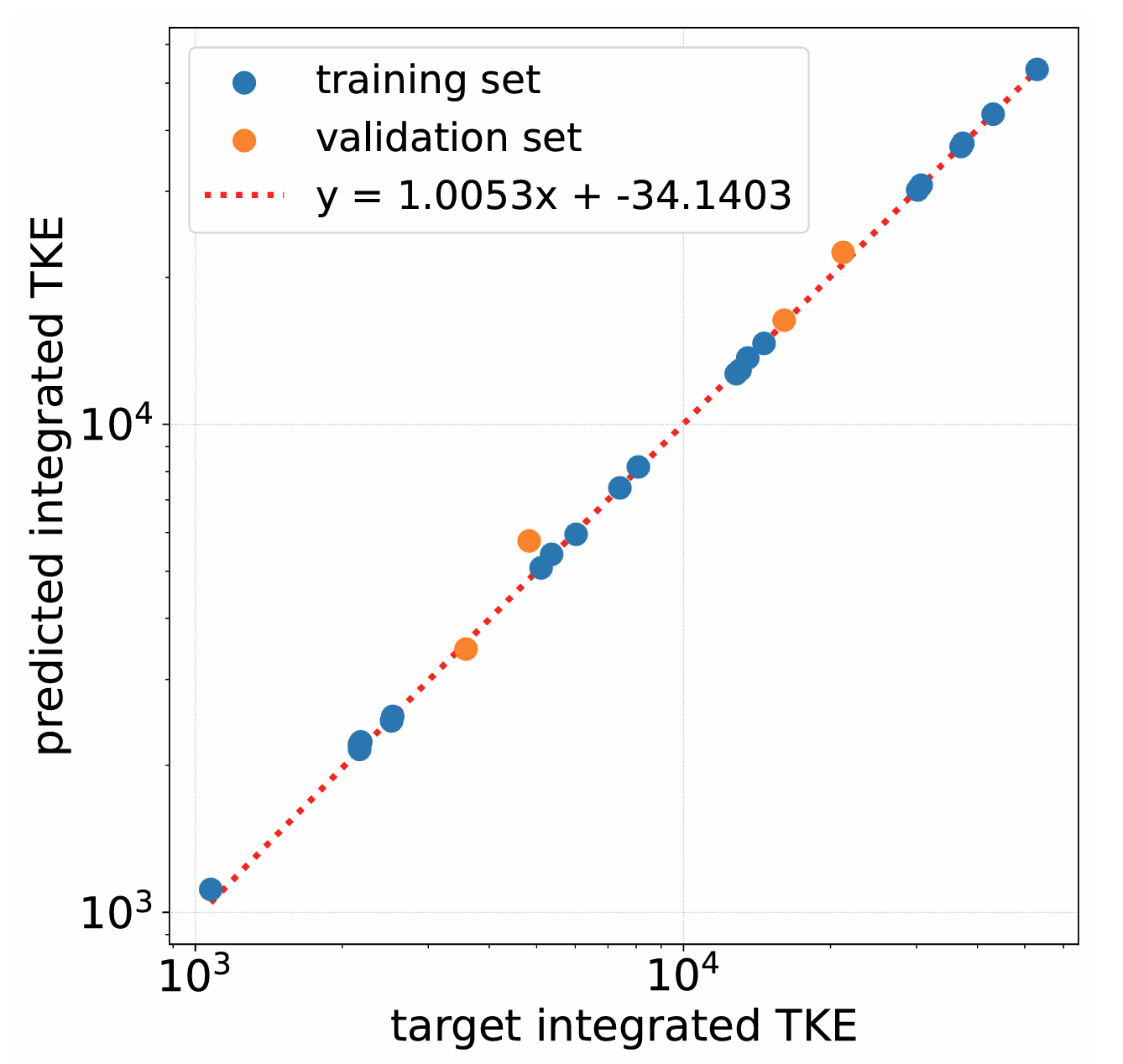}
      \caption{
        Scatter plot of each example's modeled TKE ($y$-axis) vs. the CFD simulation TKE ($x$-axis). The model accurately reproduces this integrated quantity even for validation examples.
      }
      \label{fig:integrated_TKE_scatter}
    \end{minipage}
  \end{subfigure}

  \vspace{1em} 

  \begin{subfigure}[t]{0.8\linewidth}
    \centering
    \begin{minipage}[t][8.5cm][t]{\linewidth}
      \includegraphics[width=\linewidth]{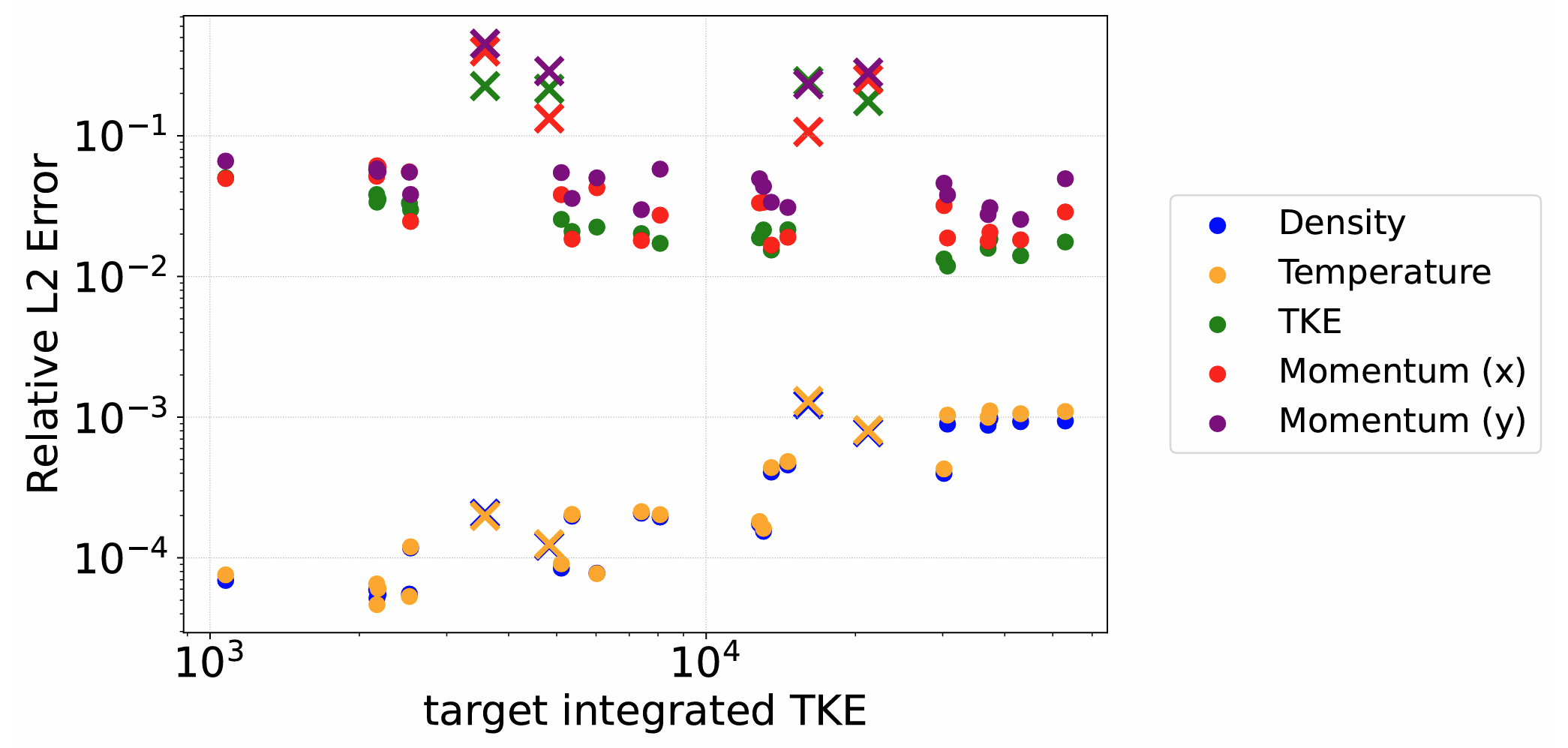}
      \caption{
        Point-wise relative $L^2$ error for all 25 simulations, organized by CFD target TKE. Data are color-coded by field; validation examples are marked with \textsc{x}'s.
      }
      \label{fig:pointwise_field_relative_L2_error}
    \end{minipage}
  \end{subfigure}

  \caption{
    Model accuracy for the turbulent kinetic energy (TKE) benchmark:
    (a) Grashof TKE scatter; (b) Integrated TKE scatter; (c) Point-wise relative $L^2$ error across all cases.
  }
  \label{fig:TKE_combined}
\end{figure}

\subsection{Conclusions}


We have presented a new framework for extracting data-driven reduced finite element exterior calculus models from data. A novel conditional transformer backbone prescribes both nonlinear finite element spaces and governing equations conditioned on a latent variable $Z$. For a suite of benchmarks we obtain highly accurate predictions generalizing across a broad range of conditioning variables with under ten basis functions for each problem, demonstrating an ability to perform real-time inference and adaptation to sensor data on consumer GPU hardware. We construct a digital twin using this framework, illustrating an ability to learn a real-time RANS-style model from a high-fidelity LES dataset consisting of few examples.

This work demonstrates that transformers do not need to be adapted in a physics/math-agnostic manner to construct AI-enabled models; the highly accurate conditioning depends crucially on the power of the transformer backbone, but by embedding within a FEEC framework we are able to provide theoretical guarantees typical of conventional finite elements. To treat unsteady problems, current techniques employ autoregressive transformer schemes. Our approach can similarly be adapted to that scenario but requires theoretical extensions which we postpone to a later work. 

\section*{Acknowledgements}
Sandia National Laboratories is a multimission laboratory managed and operated by National Technology \& Engineering Solutions of Sandia, LLC, a wholly owned subsidiary of Honeywell International Inc., for the U.S. Department of Energy’s National Nuclear Security Administration under contract DE-NA0003525. This paper describes objective technical results and analysis. Any subjective views or opinions that might be expressed in the paper do not necessarily represent the views of the U.S. Department of Energy or the United States Government. This article has been co-authored by an employee of National Technology \& Engineering Solutions of Sandia, LLC under Contract No. DE-NA0003525 with the U.S. Department of Energy (DOE). The employee owns all right, title and interest in and to the article and is solely responsible for its contents. The United States Government retains and the publisher, by accepting the article for publication, acknowledges that the United States Government retains a non-exclusive, paid-up, irrevocable, world-wide license to publish or reproduce the published form of this article or allow others to do so, for United States Government purposes. The DOE will provide public access to these results of federally sponsored research in accordance with the DOE Public Access Plan https://www.energy.gov/downloads/doe-public-access-plan.

N. Trask and B. Kinch's work is supported by the U.S. Department of Energy, Office of Advanced Scientific Computing Research under the "Scalable and Efficient Algorithms - Causal Reasoning, Operators and Graphs" (SEA-CROGS) project and the Early Career Research Program. E. Armstrong, M. Meehan, and J. Hewson's work is supported under the “Resolution-invariant deep learning for accelerated propagation of epistemic and aleatory uncertainty in multi-scale energy storage systems, and beyond” project (Project No. 81824). B. Shaffer's work is supported by the National Science Foundation Graduate Research Fellowship under Grant No. DGE-2236662.

SAND Number: SAND2025-xxxxx O

\vspace{5pt}
\noindent \textbf{Declaration of generative AI and AI-assisted technologies in the writing process:}

During the preparation of this work the author(s) used ChatGPT in order to format a complete draft, check for errors in derivations, and improve quality of typeset figures. After using this tool/service, the author(s) reviewed and edited the content as needed and take(s) full responsibility for the content of the publication.

\bibliographystyle{plain}
\bibliography{references.bib}

\appendix

\section{Architecture and training details}
\label{app:nn_details}

For both the shape function model and the flux model, we can increase the expressivity of the model simply by stacking more transformer blocks, or by using a higher dimensional embedding space. For simplicity, we use the same embedding dimension (128), the same number of attention heads (4), and the same activation functions (GELU for the shape function model, Tanh for the flux model), and the same number of transformer blocks in the flux model in each case: 3. The only hyperparameter that changes from example to example is the number of transformer blocks used in \texttt{ShapeFunctionModel}. We report this for each example shown above, along with the total learnable parameter count across both the shape function and flux neural networks, in Table \ref{tab:model_summary}.

\begin{table}[htbp]
    \centering
    \caption{Number of Transformer Blocks and Total Parameter Counts}
    \begin{tabular}{lcc}
        \toprule
        Example & \texttt{ShapeFunctionModel} $N_\textrm{blocks}$ & Total Parameter Count \\
        \midrule
        1D Advection-Diffusion & 2 & 663,747 \\
        2D Advection-Diffusion & 6 & 1,193,988 \\
        Riemann (Dirichlet) & 3 & 796,999 \\
        Riemann (Neumann) & 6 & 1,193,988 \\
        Charge in Conducting Shell & 3 & 796,999 \\
        Battery Rack & 6 & 1,193,988 \\
        \bottomrule
    \end{tabular}
    \label{tab:model_summary}
\end{table}

We use dropout in the \texttt{ShapeFunctionModel}, but not in the learnable flux model. We have found that a dropout rate of 0.1 encourages compact shape functions and fluxes between them which generalize well. Near the end of training, the dropout rate is reduced from 0.1 to 0 on a linear schedule, typically over 1000--10,000 epochs, depending on the problem. As an example, Figure \ref{fig:loss_history_confidence} shows how the training and validation loss evolve as a function of epoch for the Riemann problem with Dirichlet boundary conditions. The ``knee'' in the curves at 20,000 epochs is when the dropout scheduling begins. Before this, the characteristic compartmentalization of the problem domain (Figure \ref{fig:riemann_dir_pou_plots}) is already established---reducing the dropout rate to zero simply ``sharpens up'' these boundaries.

\begin{figure}
    \centering
    \includegraphics[width=0.5\linewidth]{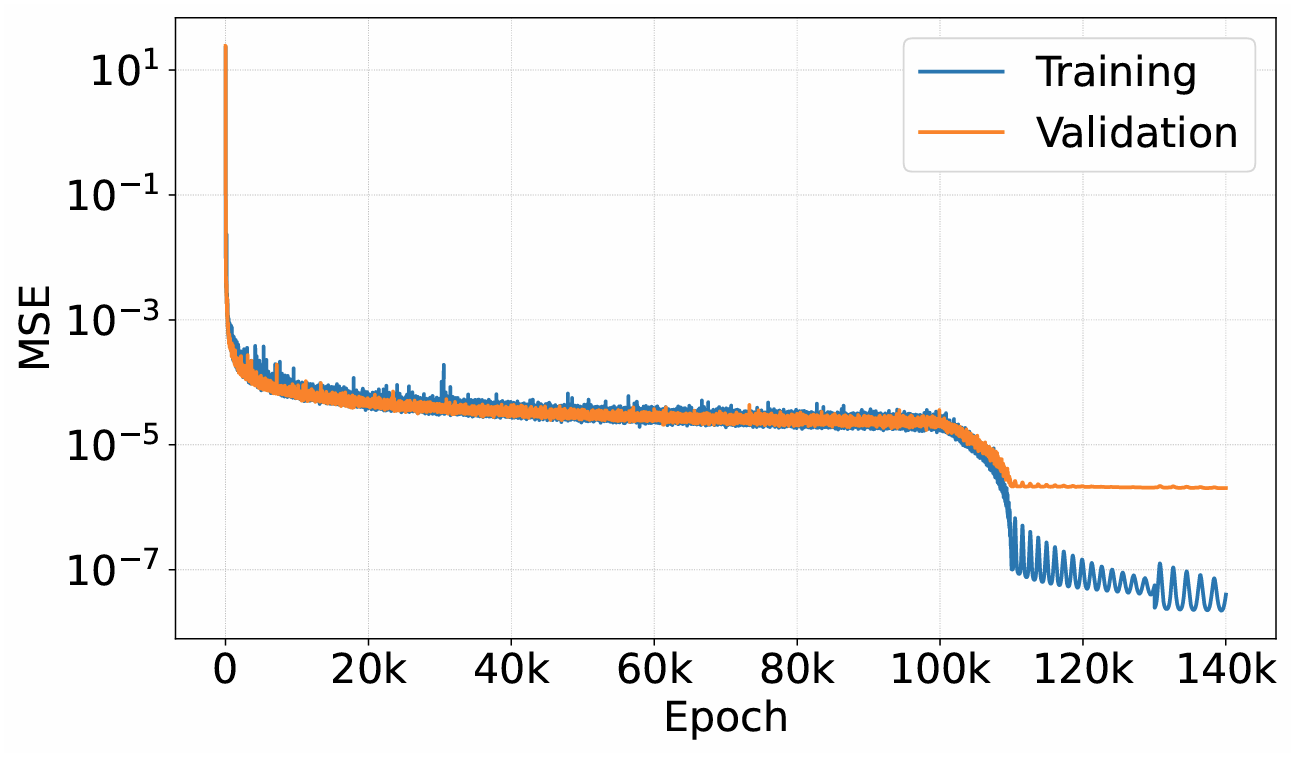}
    \caption{A typical loss curve using our method. As the dropout rate in the \texttt{ShapeFunctionModel} is scheduled down to zero at epoch 20k, the loss drops rapidly.}
    \label{fig:1d_ad_loss_curves}
\end{figure}

Finally, for all examples we use the Shampoo optimizer \citep{gupta18a, Novik_torchoptimizers}, with a learning rate of 0.1, with no momentum or weight decay; the preconditioner matrix is updated with each iteration.
\FloatBarrier
\end{document}